\RequirePackage{fix-cm}
\documentclass[smallcondensed]{svjour3}     
\smartqed  
\usepackage{graphicx}
\usepackage{epstopdf}
\usepackage{amsmath}
\usepackage{array}
\usepackage{subfigure}
\usepackage{tabularx}
\usepackage{dsfont}
\usepackage{amssymb}
\usepackage{bm}
\usepackage{url}
\usepackage{natbib}
\usepackage{yfonts}
\usepackage[colorlinks,linkcolor=blue,anchorcolor=red,citecolor=blue,urlcolor=blue]{hyperref}
\usepackage{geometry}
\DeclareMathOperator{\sgn}{sgn}
\DeclareMathOperator{\tr}{tr}
\DeclareMathOperator{\diag}{diag}
\DeclareMathOperator{\vect}{vec}
\begin{document}

\title{Multi-distance Support Matrix Machines
}


\author{Yunfei Ye$^1$ \and Dong Han$^1$
}


\institute{Yunfei Ye \\
              \email{tianshapojun@sjtu.edu.cn}  \\
              Dong Han \\
              \email{donghan@sjtu.edu.cn} \\
               $^1$ Department of Mathematical Sciences, Shanghai Jiao Tong University \\
               800 Dongchuan RD Shanghai, 200240 China \\       
}

\date{Received: date / Accepted: date}

\maketitle

\begin{abstract}
Real-world data such as digital images, MRI scans and electroencephalography signals are naturally represented as matrices with structural information. Most existing classifiers aim to capture these structures by regularizing the regression matrix to be low-rank or sparse. Some other methodologies introduce factorization technique to explore nonlinear relationships of matrix data in kernel space.  In this paper, we propose a multi-distance support matrix machine (MDSMM), which provides a principled way of solving matrix classification problems. The multi-distance is introduced to capture the correlation within matrix data, by means of an array which contains the products of columns and rows of the sample and the regression matrix. A complex hyperplane is constructed upon weighting the relative importance of the entries of the multii-distance. We further study the generalization bounds for i.i.d. processes and non i.i.d. process based on SVM, SMM and MDSMM classifiers. For typical hypothesis classes where matrix norms are constrained, MDSMM achieves a faster learning rate than traditional classifiers. We also provide a more general approach for samples without prior knowledge. We demonstrate the merits of the proposed method by conducting exhaustive experiments on both simulation study and a number of real-word datasets.

\keywords{Multi-distance support matrix machine \and Generalization bounds \and Rademacher complexity \and Vapnik-Chervonenkis dimension}
\end{abstract}

\section{Introduction}
The supervised learning tasks are often encountered in the area of machine learning, pattern recognition, image processing and data mining. The most representative method among all the traditional approaches is Support Vector Machine (SVM) \citep{vapnik2013nature}, which is originally designed for data represented as feature vectors. However, real-world data such as digital images, MRI scans and electroencephalography signals are naturally represented as matrices with structural information. Classical classifiers tend to convert matrix data into vectors, which could destroy the topological structure or result in the curse of dimensionality problem. To address these issues, researches have been exploited on classifying data in matrix form directly. \cite{gao2018multiple} proposed a multiple rank multi-linear kernel SVM (MRMLKSVM), which  introduced the left and right projecting vectors to construct decision boundary and establish margin function. Another typical structure information is the correlation between columns or rows in the data matrix which is commonly leveraged by a regression matrix. To this end, rank-$k$ SVM model \citep{wolf2007modeling} and bilinear classifiers \citep{Dyrholm2007Bilinear,Cand2009Exact} introduced certain constrains on the regression matrix. Inspired by the use of nuclear norm in low-rank matrix approximation \citep{Zhou2014Regularized}, \cite{Luo2015Support} proposed a support matrix machine (SMM) which is defined as a hinge loss plus both squared Frobenius matrix norm and nuclear norm. \cite{zheng2018sparse} proposed a sparse support matrix machine (SSMM) to involve the low-rank property and sparse property respectively. Considering that empirical EEG signals contain strong correlation information, \cite{Zheng2018Robust} assumed that each EEG matrix can be decomposed into a latent low-rank clean matrix plus a sparse noise matrix. These methods essentially take advantage of the low-rank assumption, which can be used for describing the correlation within a matrix.

Moreover, tensor methodologies can also be applied in matrix learning since matrices are second-order tensors. Several works \citep{Hao2013A,he2014dusk,Ma2016Spatio} have been presented to apply kernels methods for tensor data since the underlying structure of real data is often nonlinear. \cite{He2017Kernelized} proposed a novel Kernelized Support Tensor Machine (KSTM) which integrates kernelized tensor factorization with maximum-margin criterion.

The statistical learning theory has been extensively studied to ensure the reliability of machine learning algorithms. These algorithms enjoy a good theoretical justification in terms of universal consistency and generalization bounds under the assumption that samples are drawn i.i.d. from some unknown distribution \citep{vapnik2013nature,bartlett2002rademacher,Koltchinskii2001Rademacher,Mendelson2002Rademacher,Shalev2009Stochastic,
Srebro2010Smoothness,ying2017unregularized,liu2017infinite}. A variety of relaxations of this i.i.d. setting have been proposed in the machine learning and statistics literature. The scenario in which observations are drawn from a stationary mixing distribution (e.g., $\alpha-$mixing, $\beta-$mixing and $\phi-$mixing) has become popular and been adopted by previous studies \citep{Yu1994Rates,alquier2012model,Shalizi2013Predictive,Kuznetsov2017}. \cite{Xu2015The} studied the generalization ability of SVM based on uniformly ergodic Markov chain (u.e.M.c.) samples. The qualitative robustness of the estimator can be ensured as long as the data generating process satisfies a certain convergence condition on its empirical measure \citep{Strohriegl2016Qualitative}.

Inspired by the above work, we propose a new classifier called multi-distance support matrix machine to address the matrix classification problem. We introduce the multi-distance to explore the intrinsic information of input matrix instead of controlling low-rank and sparsity properties of the regression matrix. More specifically, the multi-distance is defined as an array which contains the products of columns and rows of the sample and the regression matrix. We apply a weight function to measure their relative importance, which is obtained by the learning algorithm. A complex hyperplane is established upon the multi-distance and the weight function to separate distinct classes. Moreover, we employ the product of the squared Frobenius norms of the regression matrix and weight function as the regularization term, to maximize the margin between matrices of different classes. We also combine the hinge loss to control the misclassification error due to its widely deployed ability.

We further present the theoretical analysis of generalization bounds for i.i.d. processes and non i.i.d. processes (stationary $\beta$-mixing, u.e.M.c. and martingale) based on SVM, SMM and MDSMM classifiers, in terms of Rademacher complexity and Vapnik-Chervonenkis dimension (VC dimension). We also provide a more general approach with a weak condition of difference between the expectation and conditional expectation of samples without prior knowledge. To demonstrate the merits of the proposed method, we conduct exhaustive experiments on both simulation study and a number of real-word datasets. The results show the effectiveness and competitiveness of MDSMM for real applications.

The rest of the paper is organized as follows. In Sect. \ref{Sect2}, we give the framework of our model and the learning algorithm. Sect. \ref{Sect3} deals with generalization bounds for both i.i.d. and non i.i.d. processes. In Sect. \ref{Sect4}, we conduct experiments to justify our methods. Finally, we conclude our remarks in Sect. \ref{cr}.

\section{Description}\label{Sect2}
In the following, we first introduce some preliminary knowledge on matrix algebra. Next, we formulate the matrix classification problem and introduce some related works. Then we proposed the multi-distance support matrix machine (MDSMM) to solve such issue, followed by the learning algorithm.

\subsection{Notations}We first introduce some basic notations and definitions. In this study, scales are denoted by lowercase letters, e.g., s, vectors by boldface lowercase letters, e.g., \textbf{v}, matrices by boldface capital letters, e.g., \textbf{M} and general sets or spaces by gothic letters, e.g., $\mathcal{S}$.

The Frobenius norm of a matrix $\textbf{A} \in \mathbb{R}^{m\times n}$ is defined by
\begin{equation}
  \| \textbf{A} \|=\sqrt{\sum_{i_1=1}^{m} \sum_{i_2=1}^{n} a_{i_1 i_2}^2},
\end{equation}
which is a generalization of the normal $\ell_2$ norm for vectors.

The inner product of two same-sized matrices $\textbf{A},\textbf{B} \in \mathbb{R}^{m\times n}$ is defined as the sum of products of their entries, i.e.,
\begin{equation}
  \langle \textbf{A},\textbf{B} \rangle=\sum_{i_1=1}^{m} \sum_{i_2=1}^{n} a_{i_1 i_2}b_{i_1 i_2}.
\end{equation}

Suppose the $p$-norm for vectors $(1\leq p \leq \infty)$ is used for both spaces $\mathbb{R}^{m}$ and $\mathbb{R}^{n}$. The induced $p$-norm on the space $\mathbb{R}^{m \times n}$ of all  matrices is defined as follows:
\begin{equation}
\|\textbf{A}\|_{p}=\sup_{\textbf{x} \neq \textbf{0}} \frac{\|\textbf{A}\textbf{x}\|_p}{\|\textbf{x}\|_p}.
\end{equation}
In the special cases of $p=1,\infty$, the induced matrix norms can be computed by $\|\textbf{A}\|_1=\max_{1\leq j \leq n}\sum_{i=1}^m |a_{ij}|$, $\|\textbf{A}\|_{\infty}=\max_{1\leq i \leq n}\sum_{j=1}^n |a_{ij}|$.

For two matrices $\textbf{A},\textbf{B} \in \mathbb{R}^{m\times n}$, the Hadamard product, $\textbf{A} \circ \textbf{B}$, is a matrix of the same dimension as the operands, with elements given by
\begin{equation}
(\textbf{A} \circ \textbf{B})_{ij}=a_{ij}b_{ij}.
\end{equation}

\subsection{Problem Formulation and Related Works}

In order to facilitate description, we first formulate the matrix classification problem as follows. Given a set of samples $\{(y_i,\textbf{X}_i)\}_{i=1}^N$ for binary classification problem, where $\textbf{X}_i \in \mathbb{R}^{m \times n}$ are the input matrix data and $y_i \in \{-1,+1\}$ are the corresponding class labels. As we have seen, $\textbf{X}_i$ is represented in matrix form. To fit a vector-based classifier, one general approach is to reshape $\textbf{X}_i$ into a vector. Then, the soft margin SVM is defined as

\begin{equation}
\min_{\textbf{w},b} \frac{1}{2}\textbf{w}^{\intercal}\textbf{w}+C\sum_{i=1}^N[1-y_i(\textbf{w}^{\intercal}\textbf{x}_i+b)]_{+},
\end{equation}
where $\textbf{x}_i=\vect(\textbf{X}_i)$, $[1-u]_{+}=\max\{0,1-u\}$ is called the hinge loss function, $\textbf{w} \in \mathbb{R}^{mn}$ is the regression parameter, $b \in \mathbb{R}$ is the offset term and $C \in \mathbb{R}$ denotes a penalty parameter.

To perform matrix data directly, \cite{Hao2013A} consider an equivalent formulation as follows.

\begin{equation}
\min_{\textbf{W},b} \frac{1}{2}\tr(\textbf{W}^{\intercal}\textbf{W})+C\sum_{i=1}^N[1-y_i(\tr(\textbf{W}^{\intercal}\textbf{X}_i)+b)]_{+},
\end{equation}
where $\textbf{W} \in \mathbb{R}^{m \times n}$. Moreover, $\tr(\textbf{W}^{\intercal}\textbf{W})=\vect(\textbf{W})^{\intercal}\vect(\textbf{W})$ and $\tr(\textbf{W}^{\intercal}\textbf{X}_i)=\vect(\textbf{W})^{\intercal}\vect(\textbf{X}_i)$ which implies that the reformulation cannot capture the correlation among columns or rows in the initial matrix.

To take the structural information into consideration, one natural approach is to consider the dependency of the regression matrix \textbf{W}. Intuitively, one can consider the following formulation

\begin{equation}
\min_{\textbf{W},b} L(\textbf{W})+P(\textbf{W}),
\end{equation}
where L(\textbf{W}) is a loss function and P(\textbf{W}) is a penalty function defined on \textbf{W}.

Since $\textbf{W}=\textbf{U}\Sigma\textbf{V}^{\intercal}$, factorization technique have been introduced to explore nonlinear relationships of matrix data in kernel space \citep{Hao2013A,he2014dusk,gao2018multiple}. Another intuitive way to leverage the structural information of matrix data is by imposing the low-rank constraint. However, determining the rank of a matrix can be NP-hard \citep{Vandenberghe1996Semidefinite} while the nuclear norm $\|\textbf{W}\|_{*}$ is best convex approximation of rank(\textbf{W}) \citep{Zhou2014Regularized}. Typically, \cite{Luo2015Support} extended elastic net penalty and suggested $P(\textbf{W})=\frac{1}{2}\tr(\textbf{W}^{\intercal}\textbf{W})+\tau\|\textbf{W}\|_{*}$ to formulate the optimization problem. \cite{zheng2018sparse} involved the low-rank property and sparse property respectively by establishing the penalty function as $P(\textbf{W})=\gamma\|\textbf{W}\|_1+\tau\|\textbf{W}\|_{*}$.

\subsection{MDSMM}
Now we introduce the multi-distance to explore the intrinsic information of input matrix. Intuitively, the multi-distance is an array which measures the distance between a data point $\textbf{X}_i$ and a hyperplane, defined by $\textbf{d}(\textbf{X}_i,\textbf{Z})$ where $\textbf{Z} \in \mathbb{R}^{m \times n}$ is a regression matrix. We further explore a weight function \textbf{g} to determine the relative importance of its entries on the average. Particularly, we present the following formulation

\begin{equation}\label{op}
\begin{split}
  &\min_{\textbf{g},\textbf{Z},\bm{\xi}} \ \frac{1}{2}\|\textbf{g}\|^2 \|\textbf{Z}\|^2 + C \sum_{i=1}^N \xi_i \\
  &s.t. \ \textbf{g}(y_i)^{\intercal}\textbf{d}(\textbf{X}_i,\textbf{Z}) \geq 1-\xi_i, \ 1 \leq i \leq N \\
  &\quad \ \ \bm{\xi} \geq 0,
\end{split}
\end{equation}
where $\bm{\xi}=[\xi_1, \cdots, \xi_N]^T$ is the vector of all slack variables of training examples, $C$ is the trade-off between the classification margin and misclassification error.

We assume that the distance function $\textbf{d}(\textbf{X}_i,\textbf{Z})$ is determined as
\begin{equation}
\textbf{d}(\textbf{X}_i,\textbf{Z})=\left(
\begin{aligned}
&\textbf{X}_i(1,:) \textbf{Z}(1,:)^{\intercal}\\
&\cdots \\
&\textbf{X}_i(m,:) \textbf{Z}(m,:)^{\intercal}\\
&\textbf{X}_i(:,1)^{\intercal} \textbf{Z}(:,1) \\
&\cdots \\
&\textbf{X}_i(:,n)^{\intercal} \textbf{Z}(:,n))^{\intercal}
\end{aligned}
\right)+\textbf{b}=\textbf{d}_2(\textbf{X}_i,\textbf{Z})+\textbf{b},
\end{equation}
where $\textbf{X}_i(1,:)$ is the $i$th row of $\textbf{X}_i$, $\textbf{X}_i(:,1)$ is the $i$th column of $\textbf{X}_i$ and $\textbf{b}=[b_1,\cdots,b_{m+n}]$ is the vector of intercepts.

For simplicity, we suppose that $\textbf{g}(y_i)=y_i \textbf{w}$ and therefore $\|\textbf{g}\|^2=\|\textbf{w}\|^2$, where $\textbf{w} \in \mathbb{R}^{m+n}$.

The Lagrangian function of the optimization problem (\ref{op}) is
\begin{equation}\label{Lp}
L(\textbf{w},\textbf{Z},\bm{\xi},\textbf{b})=\frac{1}{2}\|\textbf{w}\|^2 \|\textbf{Z}\|^2 + C \sum_{i=1}^N \xi_i -\sum_{i=1}^N \alpha_{i}(y_i \textbf{w}^{\intercal}(\textbf{d}_2(\textbf{X}_i,\textbf{Z})+\textbf{b})-1+\xi_i)
-\sum_{i=1}^N \beta_{i}\xi_{i}.
\end{equation}

In the Lagrangian function above, the only term includes \textbf{b} is $ \sum_{i=1}^N \alpha_i y_i \textbf{w}^{\intercal} \textbf{b}$. We are more concerned about the inner product $ \textbf{w}^{\intercal} \textbf{b}$ instead of $\textbf{b}$ itself, so we write the above function as
\begin{equation}\label{Lp2}
L(\textbf{w},\textbf{Z},\bm{\xi},b)=\frac{1}{2}\|\textbf{w}\|^2 \|\textbf{Z}\|^2 + C \sum_{i=1}^N \xi_i -\sum_{i=1}^N \alpha_{i}(y_i (\textbf{w}^{\intercal}\textbf{d}_2(\textbf{X}_i,\textbf{Z})+b)-1+\xi_i)
-\sum_{i=1}^N \beta_{i}\xi_{i}.
\end{equation}

Let the partial derivatives of $L(\textbf{w},\textbf{Z},\bm{\xi},b)$ with respect to
\textbf{w}, b, $\bm{\xi}$ and \textbf{Z} be zeros respectively, we have
\begin{equation}\label{Lpd}
\begin{split}
& \textbf{w}= \frac{1}{\|\textbf{Z}\|^2}\sum_{i=1}^N \alpha_i y_i \textbf{d}_2(\textbf{X}_i,\textbf{Z}), \\
& \sum_{i=1}^N \alpha_i y_i=0, \\
& \alpha_i+\beta_i=C, \ i=1,\cdots,N \\
& \textbf{Z}= \frac{1}{\|\textbf{w}\|^2}\sum_{i}^N \alpha_i  y_i \textbf{G}(\textbf{w}) \circ \textbf{X}_i,
\end{split}
\end{equation}
where $\textbf{G}(\textbf{w})=[w_i+w_{m+j}]_{m \times n}$, i.e., the $(i,j)$ entry of $\textbf{G}(\textbf{w})$ is $w_i+w_{m+j}$ for $ i= 1, \cdots, m, j=1, \cdots, n$.

Based on Eq. (\ref{Lpd}), we find the solution to \textbf{Z} depends on \textbf{w}, vice versa. That is, we cannot obtain the solution to the problem directly. The alternating projection provides a cue to have a solution. The key idea in the alternating projection optimization is to obtain the \textbf{Z} with the given \textbf{w} and obtain the \textbf{w} with the given \textbf{Z} in an iterative way. The algorithm is given in Table ~\ref{Table1}.

\begin{table}
\begin{tabular*}{1\textwidth}{@{\extracolsep{\fill}}  l  }
\hline
\noalign{\smallskip}

\noindent \textbf{Input:} The set of training data $\{\textbf{X}_i \in \mathbb{R}^{m \times n},y_i\}_{i=1}^N$, cost $C$, maximum number of loops $M$  and threshold \\ parameter $\varepsilon$ \\

\noindent \textbf{Output:} The parameters \textbf{Z}, \textbf{w} and $b$ \\

\noindent Initialization. Take $t=0$, $\textbf{w}_0=[1,\cdots,1] \in \mathbb{R}^{m+n}$, $b_0=0$ and $\textbf{Z}_0 \in \mathbb{R}^{m \times n}$ be a matrix of ones \\

\noindent \textbf{while} Convergence checking is not satisfied or the maximum number of loops $M$ is not reached \textbf{do}\\

\noindent \qquad Obtain $\textbf{w}_{t+1}$ with the given $\textbf{Z}_{t}$ by optimizing \\

\noindent \qquad  $\left[
\begin{aligned}
  &\min_{\textbf{w},b,\bm{\xi}} \ \frac{1}{2}\|\textbf{w}\|^2 \|\textbf{Z}_t\|^2 + C \sum_{i=1}^N \xi_i \\
  &s.t. \ y_i (\textbf{w}^{\intercal}\textbf{d}_2(\textbf{X}_i,\textbf{Z}_t)+b) \geq 1-\xi_i, \ 1 \leq i \leq N \\
  &\quad \ \ \bm{\xi} \geq 0,
\end{aligned}\right]$ \\

\noindent \qquad Obtain $\textbf{Z}_{t+1}, b_{t+1}$ with the given $\textbf{w}_{t+1}$ by optimizing \\

\noindent \qquad  $\left[
\begin{aligned}
  &\min_{\textbf{Z},b,\bm{\xi}} \ \frac{1}{2} \|\textbf{w}_{t+1}\|^2 \|\textbf{Z}\|^2 + C \sum_{i=1}^N \xi_i \\
  &s.t.  \ y_i (\textbf{w}_{t+1}^{\intercal}\textbf{d}_2(\textbf{X}_i,\textbf{Z})+b) \geq 1-\xi_i, \ 1 \leq i \leq N \\
  &\quad \ \ \bm{\xi} \geq 0,
\end{aligned}\right]$ \\

\noindent \qquad \textbf{Convergence checking:} \\

\noindent \qquad if $|\textbf{w}_{t+1}^{\intercal}\textbf{w}_t (\textbf{w}_{t}^{\intercal}\textbf{w}_t)^{-1}-1|+|\langle\textbf{Z}_{t+1}, \textbf{Z}_t\rangle(\langle \textbf{Z}_t, \textbf{Z}_t\rangle)^{-1}-1|< \varepsilon$, the parameters have converged \\

\noindent \qquad $t \leftarrow t+1$ \\

\noindent \textbf{end while} \\

\noindent $\textbf{w}=\textbf{w}_t, \textbf{Z}=\textbf{Z}_t, b=b_t$ \\

\hline
\end{tabular*}
\caption{}
\label{Table1}
\end{table}

Once the model has been solved, the class label of a testing example X can be predicted as follow:
\begin{equation}
y(\textbf{X})=\sgn(\textbf{w}^{\intercal}\textbf{d}_2(\textbf{X},\textbf{Z})+b).
\end{equation}

\section{Generalization Bounds}\label{Sect3}
In this section, we use Rademacher complexity and VC dimension to obtain generalization bounds with matrix norm constraint for both i.i.d. processes and non i.i.d. processes. Throughout this section, we assume that $b=0$ in the optimization problem to simplify the derivation.

\subsection{Generalization Bounds for I.I.D. Processes}\label{iidbound}

To simplify the notation, we denote
\begin{equation*}
\mathcal{F}= \ell \circ \mathcal{H}_p=\{ z \mapsto \ell(h,z): z \in \mathcal{Z}, h \in \mathcal{H}_p\},
\end{equation*}
where $\mathcal{Z}$ is a domain,  $\mathcal{H}_p$ is a hypothesis class and $\ell$ is a loss function. Given $f \in \mathcal{F}$, we define
\begin{equation*}
L_{\mathcal{D}}(f)= \mathbb{E}_{z \thicksim \mathcal{D}} [f(z)], \quad L_{\mathcal{S}}(f)=\frac{1}{N} \sum_{i=1}^N f(z_i),
\end{equation*}
where $\mathcal{D}$ is the distribution of elements in $\mathcal{Z}$, $\mathcal{S}$ is the training set and $N$ is the number of examples in $\mathcal{S}$. We present our main Rademacher complexity generalization bound for the proposed classifier.

\begin{theorem}\label{th1}
Suppose that $\mathcal{D}$ is a distribution over $\mathcal{Z}=\mathcal{X} \times \mathcal{Y}$ such that with probability 1 we have that $\|\emph{\textbf{X}} \circ \emph{\textbf{X}} \|_1 \leq R_1, \|\emph{\textbf{X}} \circ \emph{\textbf{X}} \|_{\infty} \leq R_2$. Let $\mathcal{H}_p=\{(\emph{\textbf{w}},\emph{\textbf{Z}}): \|\emph{\textbf{w}}\| \leq B, \|\emph{\textbf{Z}}\| \leq D\}$ and let $\ell: \mathcal{H}_p \times \mathcal{Z} \rightarrow \mathbb{R}$ be a loss function of the form
\begin{equation*}
\ell ((\emph{\textbf{w}},\emph{\textbf{Z}}),(\emph{\textbf{X}},y))=\Phi(\emph{\textbf{w}}^{\intercal}\emph{\textbf{d}}_2(\emph{\textbf{X}},\emph{\textbf{Z}}),y),
\end{equation*}
such that for all $y \in \mathcal{Y}$, $a \mapsto \Phi(a,y)$ is a $\rho$-Lipschitz function and $\max_{a } |\Phi(a,y)| \leq c$, where $a \in [-BD\sqrt{R_1+R_2},BD\sqrt{R_1+R_2}]$. Then, for any $\delta \in (0,1)$, with probability of at least $1-\delta$ over the choice of an i.i.d. sample of size N,
\begin{equation}\label{Bound1}
\forall \ (\emph{\textbf{w}},\emph{\textbf{Z}}) \in \mathcal{H}_p, \ L_{\mathcal{D}}((\emph{\textbf{w}},\emph{\textbf{Z}})) \leq L_{\mathcal{S}}((\emph{\textbf{w}},\emph{\textbf{Z}}))+\frac{2 \rho B D\sqrt{R_1+R_2}}{\sqrt{N}}+c\sqrt{\frac{2 \ln(2/\delta)}{N}}.
\end{equation}
\end{theorem}

\begin{proof}
See Appendix \ref{rA}.
\end{proof}
\qed

The Rademacher complexity bounds for STL framework have been discussed in \cite{Shalev2014Understanding}. We wish to estimate how good our bound is compared to a reference classifier. Under the assumptions that $\|\textbf{X}\| \leq R$, $\mathcal{H}_p=\{\textbf{W}: \|\textbf{W}\| \leq B\}$ and  let $\ell: \mathcal{H}_p \times \mathcal{Z} \rightarrow \mathbb{R}$ be a loss function of the form
\begin{equation*}
\ell (\textbf{W},(\textbf{X},y))=\Phi(\langle \textbf{W}, \textbf{X}\rangle,y),
\end{equation*}
such that for all $y \in \mathcal{Y}$, $a \mapsto \Phi(a,y)$ is a $\rho$-Lipschitz function and $\max_{a \in [-BR, BR]} |\Phi(a,y)| \leq c$. Then, for any $\delta \in (0,1)$, with probability of at least $1-\delta$ over the choice of an i.i.d. sample of size N,
\begin{equation}\label{Bound2}
\forall \ \textbf{W} \in \mathcal{H}_p, \ L_{\mathcal{D}}(\textbf{W}) \leq L_{\mathcal{S}}(\textbf{W})+\frac{2 \rho B R}{\sqrt{N}}+c\sqrt{\frac{2 \ln(2/\delta)}{N}}.
\end{equation}

If $B D\sqrt{R_1+R_2} < B' R$, we obtain that $\mathcal{A}=[-BD\sqrt{R_1+R_2},BD\sqrt{R_1+R_2}]$ is a subset of $[-B' R, B' R]$. Hence, $\max_{a \in \mathcal{A} } |\Phi(a,y)| \leq \max_{a \in [-B R', B R'] } |\Phi(a,y)|$. Plugging these inequalities into (\ref{Bound1}) and (\ref{Bound2}) and making a straightforward comparison, we conclude that:

\begin{corollary}\label{iidcompare}
Suppose that an i.i.d. process $\{\emph{\textbf{X}}_i\}_{i=1}^N$ follows the distribution $\mathcal{D}$ over $\mathcal{X} \times \mathcal{Y}$ such that with probability 1 we have that $\|\emph{\textbf{X}} \circ \emph{\textbf{X}} \|_1 \leq R_1, \|\emph{\textbf{X}} \circ \emph{\textbf{X}} \|_{\infty} \leq R_2, \|\emph{\textbf{X}}\| \leq R$. Let $\mathcal{H}_p=\{(\emph{\textbf{w}},\emph{\textbf{Z}}): \|\emph{\textbf{w}}\| \leq B, \|\emph{\textbf{Z}}\| \leq D\}$ and $\mathcal{H}'_p=\{\emph{\textbf{W}}: \|\emph{\textbf{W}}\| \leq B' \}$. Then, if $B D\sqrt{R_1+R_2} < B' R$, MDSMM obtains a smaller generalization bound than that of SVM/SMM with respect to a probability distribution.
\end{corollary}

Therefore, we have compared different bounds for the SVM/SMM classifier and proposed method constrained on different prior knowledge. We study the properties of the learning process which is sampled i.i.d. from some distribution. In the next sections, we will extend this result to the case of non i.i.d process.

\subsection{Generalization Bounds for Stationary $\beta$-mixing Processes}\label{sbmbound}

To begin with, we introduce some basic concepts based on stationary $\beta$-mixing processes, which coincides with the assumptions made in previous studies.

\begin{definition}[Stationarity]
A sequence of random variables $\{\textbf{X}_t\}_{t=-\infty}^{\infty}$, is said to be stationary if for any $t$ and non-negative integers $m$ and $k$, the random variables $(\textbf{X}_t,\cdots,\textbf{X}_{t+m})$ and $(\textbf{X}_{t+k},\cdots,\textbf{X}_{t+k+m})$ have the same distribution.
\end{definition}

\begin{definition}[$\beta$-mixing]
Let $\{\textbf{X}_t\}_{t=-\infty}^{\infty}$ be a stationary sequence of random variables. For any $i,j \in \mathbb{R} \cup \{-\infty,+\infty\}$, let $\sigma_{i}^j$ denote the $\sigma$-algebra generated by the random variables $\textbf{X}_k$, $i \leq k \leq j$. Then, for any positive integer $k$, the $\beta$-mixing coefficient of the stochastic process is defined as
\begin{equation}
\beta(k)=\sup_n \underset{B \in \sigma_{-\infty}^n}{\mathbb{E}} \bigg[\sup_{A \in \sigma^{\infty}_{n+k}} \bigg| Pr(A|B)-Pr(A) \bigg| \bigg].
\end{equation}
$\{\textbf{X}_t\}_{t=-\infty}^{\infty}$ is said to be $\beta$-mixing if $\beta(k) \rightarrow 0$. It is said to be algebraically $\beta$-mixing if there exist real numbers $\beta_0> 0$ and $r > 0$ such that $\beta(k) \leq \beta_0/k^r$ for all $k$, and exponentially mixing if there exist real numbers $\beta_0$ and $\beta_1$ such that $\beta(k) \leq \beta_0 \exp(-\beta_1k^r)$ for all $k$.
\end{definition}

Suppose we have a set of samples $\{(y_i,\textbf{X}_i)\}_{i=1}^N$ drawn from a stationary $\beta$-mixing distribution over $\mathcal{Z}=\mathcal{X} \times \mathcal{Y}$, the first step is to reduce the setting of a mixing stochastic process to a simpler scenario of a sequence of independent random variables. Then we take advantage of following concentration result introduced by \cite{Kuznetsov2017} which divides the sample into $2\mu$ blocks such that each block has size $a$.

\begin{proposition}
Let $\mathcal{H}_p$ be a set of hypotheses and let $\ell: \mathcal{H}_p \times \mathcal{Z} \rightarrow \mathbb{R}$ be a loss function bounded by $ c \geq 0$. $\mathcal{S}'=(\widetilde{\emph{\textbf{X}}}_1,\cdots,\widetilde{\emph{\textbf{X}}}_{\mu})$ where $\widetilde{\emph{\textbf{X}}}_i$, $i=1,\cdots, \mu$, are independent and each $\widetilde{\emph{\textbf{X}}}_i$ has the same distribution as $\emph{\textbf{X}}_i$.Then, for any $\mu, a > 0$ with $2\mu a = N$ and $\delta > 2(\mu - 1)\beta(a)$, with probability at least $1 - \delta$ over the choice of a stationary $\beta-$mixing sample of size N,
\begin{equation}\label{bsp}
\forall \ h \in \mathcal{H}_p, \ L_{\mathcal{D}}(h) \leq L_{\mathcal{S}}(h)+2 \underset{\mathcal{S}'}{\mathbb{E}} R(\ell \circ \mathcal{H}_p \circ \mathcal{S}')+c\sqrt{\frac{ \ln(2/\delta')}{2\mu}},
\end{equation}
where $\delta'=\delta-2(\mu-1)\beta(a)$.
\end{proposition}

Similarly, the Rademacher complexity $\underset{\mathcal{S}'}{\mathbb{E}} R(\ell \circ \mathcal{H}_p \circ \mathcal{S}')$ can be calculated as in Sect. \ref{iidbound}.  Intuitively, we obtain the following bound constrained on the hypothesis class $\mathcal{H}_p=\{(\textbf{w},\textbf{Z}): \|\textbf{w}\| \leq B, \|\textbf{Z}\| \leq D\}$ for MDSMM
\begin{equation}
\underset{\mathcal{S}'}{\mathbb{E}} R(\ell \circ \mathcal{H}_p \circ \mathcal{S}') \leq \frac{\rho B D \sqrt{R_1+R_2}}{\sqrt{\mu}},
\end{equation}
where $\|\textbf{X} \circ \textbf{X} \|_1 \leq R_1$ and $\|\textbf{X} \circ \textbf{X} \|_{\infty} \leq R_2$. Plugging this into (\ref{bsp}) yields

\begin{theorem}\label{statheo}
Suppose that $\mathcal{D}$ is a stationary $\beta$-mixing distribution over $\mathcal{Z}=\mathcal{X} \times \mathcal{Y}$ such that with probability 1 we have that $\|\emph{\textbf{X}} \circ \emph{\textbf{X}} \|_1 \leq R_1, \|\emph{\textbf{X}} \circ \emph{\textbf{X}} \|_{\infty} \leq R_2$. Let $\mathcal{H}_p=\{(\emph{\textbf{w}},\emph{\textbf{Z}}): \|\emph{\textbf{w}}\| \leq B, \|\emph{\textbf{Z}}\| \leq D\}$ and let $\ell: \mathcal{H}_p \times \mathcal{Z} \rightarrow \mathbb{R}$ be a loss function of the form
\begin{equation*}
\ell ((\emph{\textbf{w}},\emph{\textbf{Z}}),(\emph{\textbf{X}},y))=\Phi(\emph{\textbf{w}}^{\intercal}\emph{\textbf{d}}_2(\emph{\textbf{X}},\emph{\textbf{Z}}),y),
\end{equation*}
such that for all $y \in \mathcal{Y}$, $a \mapsto \Phi(a,y)$ is a $\rho$-Lipschitz function and $\max_{a } |\Phi(a,y)| \leq c$, where $a \in [-BD\sqrt{R_1+R_2},BD\sqrt{R_1+R_2}]$. Then, for any $\mu, a > 0$ with $2\mu a = N$ and $\delta > 2(\mu - 1)\beta(a)$, with probability at least $1 - \delta$ over the choice of a stationary $\beta-$mixing sample of size N,
\begin{equation}\label{Bound3}
\forall \ h \in \mathcal{H}_p, \ L_{\mathcal{D}}(h) \leq L_{\mathcal{S}}(h)+\frac{2 \rho B D\sqrt{R_1+R_2}}{\sqrt{\mu}}+c\sqrt{\frac{\ln(2/\delta')}{2\mu}},
\end{equation}
where $\delta'=\delta-2(\mu-1)\beta(a)$.
\end{theorem}

Notice that $\mathcal{S}'$ is a subset of $\mathcal{S}$ and $\frac{1}{\mu}>\frac{1}{N}$ which implies that the stationary $\beta$-mixing process achieves a slower convergence than that of i.i.d. samples. The comparison of the generalization bounds of MDSMM and SVM/SMM will be further discussed in Sect. \ref{gbmp}.

\subsection{Generalization Bounds for u.e.M.c. Samples}

In this section, we evaluate the generalization bounds for uniformly ergodic Markov chain (u.e.M.c.) samples, since the u.e.M.c. is a natural representation for non i.i.d. process in real-world problems. Our derivation relies on the technique of VC dimension. First, we introduce some basic definitions of u.e.M.c.

Given a set of samples $\{Z_t\}_{t \geq 1}$ in a measurable space $(\mathcal{Z},\mathcal{S})$, for $\mathcal{A} \subseteq \mathcal{S}, z_i \in \mathcal{Z}$ it is assumed that
\begin{equation*}
Pr^n(\mathcal{A}|z_i) :=Pr(Z_{n+i} \in \mathcal{A}| Z_j, j<i, Z_i=z_i).
\end{equation*}

A Markov chain $\{Z_t\}_{t \geq 1}$ is a sequence of random variables with the Markov property, that is $Pr^n(\mathcal{A}|z_i)=Pr(Z_{n+i} \in \mathcal{A}|Z_i=z_i)$. Given two probabilities $\nu_1, \nu_2$ on the space $(\mathcal{Z},\mathcal{S})$, the total variation distance between $\nu_1, \nu_2$ is defined as $\|\nu_1-\nu_2\|_{TV}=\sup_{\mathcal{A} \subseteq \mathcal{S}} |\nu_1(\mathcal{A})-\nu_2(\mathcal{A})|$. Thus, we have the following definition of u.e.M.c \citep{vidyasagar2013learning}.

\begin{definition}
A Markov Chain $\{Z_t\}_{t \geq 1}$ is said to be uniformly ergodic if for some $0 < \gamma_0 < \infty$ and $0 < \rho_0 <1$
\begin{equation*}
\|Pr^k(\cdot|z)-\pi(\cdot)\|_{TV} \leq \gamma_0 \rho_0^k, \forall k \geq 1, k \in \mathbb{N}
\end{equation*}
where $\pi(\cdot)$ is the stationary distribution of $\{Z_t\}_{t \geq 1}$.
\end{definition}

Furthermore, the  $k$-step transition probability measure $Pr^k(\cdot|\cdot)$ of u.e.M.c is proved to satisfy the Doeblin condition \citep{Doukhan1994Mixing} by \cite{Meyn1993Markov}.

\begin{proposition}[Doeblin condition]\label{Dc}
Let $\{Z_t\}_{t \geq 1}$ be a Markov chain with transition probability measure $Pr^k(\cdot|\cdot)$ and $\mu$ be a nonnegative measure with nonzero mass $\mu_0$. If there exists a integer t such that for all $z \in \mathcal{Z}$ and measurable set $\mathcal{A}$, $Pr^t(\mathcal{A}|z) \leq \mu(\mathcal{A})$, then for any integer k and $z, z' \in \mathcal{Z}$
\begin{equation}
\bigg\| Pr^k(\cdot|z)-Pr^k(\cdot|z')\bigg\|_{TV} \leq 2\beta_1^{k/t},
\end{equation}
where $\beta_1=1-\mu_0$
\end{proposition}

Given these conditions, \cite{Xu2015The} obtained the generalization bound for time series prediction with a u.e.M.c. process. That is
\begin{proposition}\label{Xu2015}
Let $\mathcal{F}$ be a countable class of bounded measurable function and $\{Z_t\}_{t =1}^N$ be a u.e.M.c. sample. Assume that $0 \leq g(z) \leq c$ for all $g \in \mathcal{F}$ and $z \in \mathcal{Z}$. Then for any $\varepsilon>0$
\begin{equation}\label{Xulemma}
Pr \bigg\{ \bigg|\mathbb{E}(g) -\frac{1}{N} \sum_{i=1}^N g(z_i) \bigg| \geq \varepsilon  \bigg\} \leq 2\exp \bigg\{\frac{-N \varepsilon^2}{56c \|\Gamma_0\|^2 \mathbb{E}(g) }\bigg\},
\end{equation}
where $\|\Gamma_0\|=\sqrt{2}/(1-\beta_1^{1/2t})$ and $\beta_1$ and $t$ are defined in Proposition \ref{Dc}.
\end{proposition}

The preceding proposition enables us to show the main result in this section.

\begin{theorem}\label{mctheo}
Suppose that a u.e.M.c. process $\{\emph{\textbf{X}}_i,y_i\}_{i=1}^N$ follows the distribution $\mathcal{D}$ over $\mathcal{Z}=\mathcal{X} \times \mathcal{Y}$. We establish another u.e.M.c. sample $\mathcal{S}'=\{\emph{\textbf{X}}'_i,y'_i\}_{i =1}^N$ which is independent of $\{\emph{\textbf{X}}_i,y_i\}_{i=1}^N$. Let $\mathcal{H}_p$ be a set of hypotheses and let $\ell: \mathcal{H}_p \times \mathcal{Z} \rightarrow \mathbb{R}$ be a binary loss function bounded by $ c \geq 0$. Then, for any $\delta > 0$ and $d_{\mathcal{G}}\ln{(\frac{2eN}{d_{\mathcal{G}}})}\geq \ln{\delta}-\frac{5}{2}\ln{2}$, with probability at least $1 - \delta$ over the choice of a u.e.M.c. sample of size N,
\begin{equation}
\forall \ h \in \mathcal{H}_p, \ L_{\mathcal{D}}(h) \leq L_{\mathcal{S}}(h)+8\sqrt{14}c\|\Gamma_0\| \sqrt{\frac{3}{N}\ln{2}+\frac{d_{\mathcal{G}}}{N}\ln{(\frac{2eN}{d_{\mathcal{G}}})}+\frac{ \ln{(1/\delta)}}{N}},
\end{equation}
where $\mathcal{G}=\{f(\emph{\textbf{X}}_1,y_1),\cdots,f(\emph{\textbf{X}}_n,y_n),f(\emph{\textbf{X}}'_1,y'_1),\cdots,f(\emph{\textbf{X}}'_n,y'_n)\}$, $f=\ell \circ h$ and $d_{\mathcal{G}}$ is its VC dimension.

\end{theorem}

\begin{proof}
Combine the theories of VC dimension and Proposition \ref{Xu2015}, for $\exp \bigg\{\frac{-N \varepsilon^2}{224c^2 \|\Gamma_0\|^2 }\bigg\} \leq \frac{1}{4}$ we obtain
\begin{flalign*}
&Pr \bigg\{\sup_{h \in \mathcal{H}_p} |L_{\mathcal{D}}(h)-L_{\mathcal{S}}(h)|\geq \varepsilon \bigg\} \\
&\leq 2 Pr \bigg\{\sup_{h \in \mathcal{H}_p} |L_{\mathcal{S}}(h)-L_{\mathcal{S'}}(h)|\geq\frac{\varepsilon}{2} \bigg\} &(symmetrization)\\
&\leq 2 m_{\mathcal{G}}(2N) Pr \bigg\{|L_{\mathcal{S}}(h)-L_{\mathcal{S'}}(h)|\geq\frac{\varepsilon}{2} \bigg| h \in \mathcal{H}_p \bigg\} &(union \ bound)\\
&\leq 4 m_{\mathcal{G}}(2N) Pr \bigg\{|L_{\mathcal{D}}(h)-L_{\mathcal{S'}}(h)|\geq\frac{\varepsilon}{4} \bigg| h \in \mathcal{H}_p \bigg\} &(union \ bound)\\
& \leq 8 m_{\mathcal{G}}(2N) \exp \bigg\{\frac{-N \varepsilon^2}{896c^2 \|\Gamma_0\|^2 }\bigg\} &(use \ of \ (\ref{Xulemma}))\\
& \leq 8 (\frac{2eN}{d_{\mathcal{G}}})^{d_{\mathcal{G}}} \exp \bigg\{\frac{-N \varepsilon^2}{896c^2 \|\Gamma_0\|^2 }\bigg\}, &(Sauer-Shelah \ lemma)
\end{flalign*}
where $ m_{\mathcal{G}}(2N) $ is the growth function of the space $\mathcal{G}$.

Therefore, set the right-hand side of above inequality to $\delta$ shows the result.
\end{proof}
\qed

Thus, if $d_{\mathcal{G}}<\infty$, the algorithm is learnable and the excess risk is bounded by $O(\sqrt{\frac{1}{N}\ln{N}})$ apart from the constant term.

\subsection{Generalization Bounds for Martingale Processes}\label{gbmp}

In this section we would like to analyze the Rademacher complexity of martingale processes. Below we will show that the expected value of the largest gap between the true error and its empirical error is bounded by twice the expected Rademacher complexity \citep{Shalev2014Understanding}.

\begin{lemma}\label{marineq}
$\underset{\mathcal{S} \sim \mathcal{D}^N}{\mathbb{E}}[\sup_{h \in \mathcal{H}_p}{(L_{\mathcal{D}}(h)-L_{\mathcal{S}}(h))}] \leq 2 \underset{\mathcal{S} \sim \mathcal{D}^N}{\mathbb{E}} R(\ell \circ \mathcal{H}_p \circ \mathcal{S}).$
\end{lemma}

Combining Lemma \ref{marineq} and Azuma's inequality leads directly to the main result of this section.

\begin{theorem}\label{martheo}
Let $\ell: \mathcal{H}_p \times \mathcal{Z} \rightarrow \mathbb{R}$ be a nonnegative loss function bounded by $ c \geq 0$, where $\mathcal{Z}=\mathcal{X} \times \mathcal{Y}$ and $\mathcal{H}_p$ is a set of hypotheses. Suppose that $\mathcal{D}$ is a distribution over $\mathcal{Z}$ such that with probability 1 we have that $\{\emph{\textbf{Y}}_i=\sup_{h \in \mathcal{H}_p}\sum\limits_{j=1}^i (\mathbb{E}[f(z_j)]-f(z_j) ) - \mathbb{E}[\sup_{h \in \mathcal{H}_p}\sum\limits_{j=1}^i (\mathbb{E}[f(z_j)]-f(z_j) )] \}_{i=1}^N$ is a martingale with respect to $\{\emph{\textbf{X}}_i\}_{i=1}^N$, where $f=\ell \circ h$, $z_j=(\emph{\textbf{X}}_j,y_j)$ and $\emph{\textbf{Y}}_0=0$. Then, for any $\delta \in (0,1)$, with probability at least $1 - \delta$ over the choice of an martingale sample of size N,
\begin{equation}
\forall \ h \in \mathcal{H}_p, \ L_{\mathcal{D}}(h) \leq L_{\mathcal{S}}(h)+2 \underset{\mathcal{S}}{\mathbb{E}} R(\ell \circ \mathcal{H}_p \circ \mathcal{S})+2c\sqrt{\frac{ 2\ln(1/\delta)}{N}}.
\end{equation}
\end{theorem}

\begin{proof}
Using Azuma's inequality we only need to show that $\textbf{Y}_i$ has bounded differences. Thus, we obtain that
\begin{equation*}
\begin{split}
|\textbf{Y}_i-\textbf{Y}_{i-1}| &\leq | \sup_{h \in \mathcal{H}_p}\sum\limits_{j=1}^i (\mathbb{E}[f(z_j)]-f(z_j) ) - \sup_{h \in \mathcal{H}_p}\sum\limits_{j=1}^{i-1} (\mathbb{E}[f(z_j)]-f(z_j) ) | \\
&+ \mathbb{E} [| \sup_{h \in \mathcal{H}_p}\sum\limits_{j=1}^i (\mathbb{E}[f(z_j)]-f(z_j) ) - \sup_{h \in \mathcal{H}_p}\sum\limits_{j=1}^{i-1} (\mathbb{E}[f(z_j)]-f(z_j) ) |].
\end{split}
\end{equation*}
Next, we note that for all $h \in \mathcal{H}_p$ we obtain
\begin{equation*}
| \sum\limits_{j=1}^i (\mathbb{E}[f(z_j)]-f(z_j) )-\sum\limits_{j=1}^{i-1} (\mathbb{E}[f(z_j)]-f(z_j) ) | = |\mathbb{E}[f(z_i)]-f(z_i)| \leq c
\end{equation*}
Therefore, $| \sup_{h \in \mathcal{H}_p}\sum\limits_{j=1}^i (\mathbb{E}[f(z_j)]-f(z_j) ) - \sup_{h \in \mathcal{H}_p}\sum\limits_{j=1}^{i-1} (\mathbb{E}[f(z_j)]-f(z_j) ) | \leq c$ and $|\textbf{Y}_i-\textbf{Y}_{i-1}| \leq 2c$.

Combining Lemma \ref{marineq}, Azuma's inequality and the above results we obtain that with probability of at least $1-\delta$,
\begin{equation}
\begin{split}
\sup_{h \in \mathcal{H}_p}{(L_{\mathcal{D}}(h)-L_{\mathcal{S}}(h))} &\leq \underset{\mathcal{S} \sim \mathcal{D}^N}{\mathbb{E}}[\sup_{h \in \mathcal{H}_p}{(L_{\mathcal{D}}(h)-L_{\mathcal{S}}(h))}] + 2c\sqrt{\frac{2\ln{(1/\delta})}{N}} \\
& \leq 2 \underset{\mathcal{S} \sim \mathcal{D}^N}{\mathbb{E}} R(\ell \circ \mathcal{H}_p \circ \mathcal{S}) + 2c\sqrt{\frac{2\ln{(1/\delta)}}{N}}.
\end{split}
\end{equation}

\end{proof}
\qed

Next we are willing to compare the generalization bounds for SVM/SMM and MDSMM. Suppose that a set of samples $\{\textbf{X}_i\}_{i=1}^N$ follows the distribution $\mathcal{D}$ over $\mathcal{Z}=\mathcal{X} \times \mathcal{Y}$ such that with probability 1 we have that $\|\textbf{X} \circ\textbf{X} \|_1 \leq R_1, \|\textbf{X} \circ \textbf{X} \|_{\infty} \leq R_2$ and $\|\textbf{X}\| \leq R$. The Rademacher complexity $\underset{\mathcal{S}}{\mathbb{E}} R(\ell \circ \mathcal{H}_p \circ \mathcal{S})$ that appears in our bound is standard. For instance, it can be bounded under the framework of SVM/SMM classifier containing the set of linear hypotheses $\mathcal{H}'_p=\{\textbf{W}: \|\textbf{W}\| \leq B' \}$ by
\begin{equation*}
\underset{\mathcal{S}}{\mathbb{E}} R(\ell \circ \mathcal{H}_p \circ \mathcal{S}) \leq \frac{\rho B' R}{\sqrt{N}}
\end{equation*}
for $\rho$-Lipschitz losses. Alternatively, the MDSMM is established on another set of hypotheses $\mathcal{H}_p=\{(\textbf{w},\textbf{Z}): \|\textbf{w}\| \leq B, \|\textbf{Z}\| \leq D\}$ which yields the following bound:
\begin{equation*}
\underset{\mathcal{S}}{\mathbb{E}} R(\ell \circ \mathcal{H}_p \circ \mathcal{S}) \leq \frac{\rho B D \sqrt{R_1+R_2}}{\sqrt{N}}.
\end{equation*}



We repeat the symbols and definitions in Sect. \ref{iidbound} to compare the generalization bounds. We write the upper bounds for loss functions in the SVM/SMM and MDSMM as
\begin{equation*}
\begin{split}
&c_1=\max_{(\textbf{w},\textbf{Z}) \in \mathcal{H}_p} \Phi(\textbf{w}^{\intercal}\textbf{d}_2(\textbf{X},\textbf{Z}),y),\\
&c_2=\max_{\textbf{W} \in \mathcal{H}'_p} \Phi(\langle \textbf{W}, \textbf{X}\rangle,y),
\end{split}
\end{equation*}
where $\mathcal{H}_p=\{(\textbf{w},\textbf{Z}): \|\textbf{w}\| \leq B, \|\textbf{Z}\| \leq D\}$ and $\mathcal{H}'_p=\{\textbf{W}: \|\textbf{W}\| \leq B' \}$. Thus, it is straightforward to rewrite $c_1$ and $c_2$ as follows:
\begin{equation*}
\begin{split}
&c_1=\max\nolimits_{a \in [-BD\sqrt{R1+R_2},BD\sqrt{R1+R_2}]} \Phi(a,y),\\
&c_2=\max\nolimits_{a \in [-B'R,B'R]} \Phi(a,y).
\end{split}
\end{equation*}

It is trivial to observe that if $B D\sqrt{R_1+R_2} < B' R$, then $c_1 \leq c_2$. Plugging these inequalities into Theorem \ref{statheo}, \ref{mctheo} and \ref{martheo} we conclude that:

\begin{corollary}\label{niidcompare}
Suppose that $\{\emph{\textbf{X}}_i\}_{i=1}^N$ is a stationary $\beta$-mixing process (u.e.M.c.) or $\{\emph{\textbf{Y}}_i\}_{i=1}^N$  is a martingale follows the distribution $\mathcal{D}$ over $\mathcal{X} \times \mathcal{Y}$ such that with probability 1 we have that $\|\emph{\textbf{X}} \circ \emph{\textbf{X}} \|_1 \leq R_1, \|\emph{\textbf{X}} \circ \emph{\textbf{X}} \|_{\infty} \leq R_2, \|\emph{\textbf{X}}\| \leq R$. Let $\mathcal{H}_p=\{(\emph{\textbf{w}},\emph{\textbf{Z}}): \|\emph{\textbf{w}}\| \leq B, \|\emph{\textbf{Z}}\| \leq D\}$ and $\mathcal{H}'_p=\{\emph{\textbf{W}}: \|\emph{\textbf{W}}\| \leq B' \}$. Then, if $B D\sqrt{R_1+R_2} < B' R$, MDSMM achieves a smaller generalization bound than that of SVM/SMM with respect to a probability distribution.
\end{corollary}

We note that the above formulation assumes that $\textbf{Y}_i$ is a martingale, which is a rather strong assumption. We will discuss in the following section about a more relaxed, weaker condition for the generalization bounds.

\subsection{Generalization Bounds for Stochastic Processes}
In the previous sections we derived generalization bounds for i.i.d. process and  non i.i.d. processes with constrains on the sample spaces. Now we would like to take a more general approach, and aim at deriving bound regardless of the properties of samples.

We establish our bounds in a series of lemmas. To simplify the derivation we start by proving a natural result obtained by Azuma's inequality.
\begin{lemma}\label{spl1}
Suppose that $\{\emph{\textbf{X}}_i,y_i\}_{i=1}^N$ are N random variables taking values in $\mathcal{Z}=\mathcal{X} \times \mathcal{Y}$. Let $\mathcal{H}_p$ be a set of hypotheses and let $\ell: \mathcal{H}_p \times \mathcal{Z} \rightarrow \mathbb{R}$ be a loss function bounded by $ c \geq 0$ for all $h \in \mathcal{H}_p$.
Then, for any $\delta \in (0,1)$, with probability at least $1 - \delta$ over the choice of a sample of size N,
\begin{equation}
|\widetilde{L}_{\mathcal{D}}(h)-L_{\mathcal{S}}(h)| \leq  2c\sqrt{\frac{ 2\ln(2/\delta)}{N}},
\end{equation}
where $\widetilde{L}_{\mathcal{D}}(h)=\frac{1}{N}\sum_{i=1}^N \mathbb{E}[f(z_i) | z_j, j < i]$, $z_i=(\emph{\textbf{X}}_i,y_i)$ and $f=\ell \circ h$.
\end{lemma}

\begin{proof}
The main idea follows from the fact that $\{\textbf{Y}_i=\frac{1}{N}\sum_{j=1}^i ( \mathbb{E}[f(z_j) | z_k, k < j]-f(z_j) ) \}_{i=1}^N$ is a martingale with respect to $\{z_i\}_{i=1}^N$. That is,
\begin{equation}
\begin{split}
&\mathbb{E} [|\textbf{Y}_i|] \leq 2c \\
&\mathbb{E}[\textbf{Y}_i| z_k, k < i]=\frac{1}{N}(\sum_{j=1}^i \mathbb{E}[f(z_j) | z_k, k < j]-\sum_{j=1}^{i-1} f(z_j)- \mathbb{E}[f(z_i) | z_k, k < i])=\textbf{Y}_{i-1}.
\end{split}
\end{equation}
Define $\textbf{Y}_0=\mathbb{E}[\textbf{Y}_1]=0$. In addition, the differeces between $\textbf{Y}_i$ and $\textbf{Y}_{i-1}$ can be calculated as
\begin{equation*}
|\textbf{Y}_i-\textbf{Y}_{i-1}|=\frac{1}{N}|\mathbb{E}[f(z_i) | z_k, k < i]-f(z_i)| \leq \frac{2c}{N}
\end{equation*}

Hence, using Azuma's inequality we obtain that for all $\varepsilon >0$,
\begin{equation}
Pr\{|\textbf{Y}_N-\textbf{Y}_0| \geq \varepsilon\}=Pr\{|\widetilde{L}_{\mathcal{D}}(h)-L_{\mathcal{S}}(h)| \geq \varepsilon\} \leq 2\exp \bigg\{ \frac{-N\varepsilon^2}{8c^2}\bigg\}
\end{equation}
Set the right-hand side of inequality to $\delta$ concludes our proof.
\end{proof}
\qed

Note that we do not enforce the hard constrains on the properties of variables. Indeed, the above lemma is expected to be applied for non-identically distributed variables. However, $\widetilde{L}_{\mathcal{D}}(h)$ is a conditional expectation and hence is itself a random variable. We would like obtain an upper bound between the true expectation and the conditional expectation. A natural idea is to implement a prior assumption, that is:

\begin{lemma}\label{spl2}
Use the notation of Lemma \ref{spl1}, suppose that there exists $\widetilde{c}>0$ satisfying $\ln{(\widetilde{c})}/N \rightarrow 0$ as $N \rightarrow \infty$ such that for any $\varepsilon>0$ we have
\begin{equation}\label{spassump}
Pr\{|L_{\widetilde{\mathcal{D}}}(h)-\widetilde{L}_{\mathcal{D}}(h)| \geq \varepsilon\} \leq \widetilde{c} \exp \bigg\{ \frac{-N\varepsilon^2}{8c^2}\bigg\},
\end{equation}
where $L_{\widetilde{\mathcal{D}}}(h)=\frac{1}{N}\sum_{i=1}^N \mathbb{E}[f(z_i)]$. Then,
\begin{equation}\label{spassump2}
Pr\{|L_{\widetilde{\mathcal{D}}}(h)-L_{\mathcal{S}}(h)| \geq \varepsilon\} \leq (\widetilde{c}+2) \exp \bigg\{ \frac{-N\varepsilon^2}{32c^2}\bigg\}.
\end{equation}
\end{lemma}

\begin{proof}
It is trivial to see that $Pr\{|L_{\widetilde{\mathcal{D}}}(h)-L_{\mathcal{S}}(h)| \geq \varepsilon\} \leq Pr\{|L_{\widetilde{\mathcal{D}}}(h)-\widetilde{L}_{\mathcal{D}}(h)| \geq \varepsilon/2\}+Pr\{|\widetilde{L}_{\mathcal{D}}(h)-L_{\mathcal{S}}(h)| \geq \varepsilon/2\}$. Then we can combine Eq. (\ref{spassump}) with Lemma \ref{spl1} to obtain the result.
\end{proof}
\qed

Equipped with the preceding lemmas we are now ready to state and prove the
main result of this section: an upper bound on the generalization error.

\begin{theorem}\label{stob}
Use the notation of Lemma \ref{spl1} and suppose that Eq. (\ref{spassump}) holds for a binary loss function $\ell$. Let $\mathcal{S}'=\{\emph{\textbf{X}}'_i,y'_i\}_{i =1}^N$ be another set of samples which is independent and follows the same distribution of $\mathcal{S}$. Then, for any $\delta > 0$ and $d_{\mathcal{G}}\ln{(\frac{2eN}{d_{\mathcal{G}}})}\geq \ln{\delta}-\frac{7}{4}\ln{2}-\frac{3}{4}\ln{(\widetilde{c}+2)}$, with probability at least $1 - \delta$ over the choice of a sample of size N,
\begin{equation}
\forall \ h \in \mathcal{H}_p, \ L_{\widetilde{\mathcal{D}}}(h) \leq L_{\mathcal{S}}(h)+16\sqrt{2}c \sqrt{\frac{2}{N}\ln{2}+\frac{\ln{(\widetilde{c}+2)}}{N}+\frac{d_{\mathcal{G}}}{N}\ln{(\frac{2eN}{d_{\mathcal{G}}})}+\frac{ \ln{(1/\delta)}}{N}},
\end{equation}
where $\mathcal{G}=\{f(\emph{\textbf{X}}_1,y_1),\cdots,f(\emph{\textbf{X}}_n,y_n),f(\emph{\textbf{X}}'_1,y'_1),\cdots,f(\emph{\textbf{X}}'_n,y'_n)\}$, $f=\ell \circ h$ and $d_{\mathcal{G}}$ is its VC dimension.

\end{theorem}

\begin{proof}
Combine the theories of VC dimension and Lemma \ref{spl2}, for $(\widetilde{c}+2)\exp \bigg\{\frac{-N \varepsilon^2}{128c^2}\bigg\} \leq \frac{1}{2}$ we obtain
\begin{flalign*}
&Pr \bigg\{\sup_{h \in \mathcal{H}_p} |L_{\widetilde{\mathcal{D}}}(h)-L_{\mathcal{S}}(h)|\geq \varepsilon \bigg\} \\
&\leq 2 Pr \bigg\{\sup_{h \in \mathcal{H}_p} |L_{\mathcal{S}}(h)-L_{\mathcal{S'}}(h)|\geq\frac{\varepsilon}{2} \bigg\} &(symmetrization)\\
&\leq 2 m_{\mathcal{G}}(2N) Pr \bigg\{|L_{\mathcal{S}}(h)-L_{\mathcal{S'}}(h)|\geq\frac{\varepsilon}{2} \bigg| h \in \mathcal{H}_p \bigg\} &(union \ bound)\\
&\leq 4 m_{\mathcal{G}}(2N) Pr \bigg\{|L_{\widetilde{\mathcal{D}}}(h)-L_{\mathcal{S'}}(h)|\geq\frac{\varepsilon}{4} \bigg| h \in \mathcal{H}_p \bigg\} &(union \ bound)\\
& \leq 4(\widetilde{c}+2) m_{\mathcal{G}}(2N) \exp \bigg\{\frac{-N \varepsilon^2}{512c^2}\bigg\} &(use \ of \ (\ref{spassump2}))\\
& \leq 4(\widetilde{c}+2) (\frac{2eN}{d_{\mathcal{G}}})^{d_{\mathcal{G}}} \exp \bigg\{\frac{-N \varepsilon^2}{512c^2 }\bigg\}, &(Sauer-Shelah \ lemma)
\end{flalign*}
where $ m_{\mathcal{G}}(2N) $ is the growth function of the space $\mathcal{G}$.

Therefore, set the right-hand side of above inequality to $\delta$ shows the result.
\end{proof}
\qed

\begin{remark}
The proceeding theorem tells us in what conditions that the true error of the learned predictor will be bounded by its empirical error plus $\varepsilon$. For example, if $\{\textbf{X}_i\}_{i=1}^N$ is i.i.d or $\{\sum_{j=1}^i (\mathbb{E}[f(z_j)]-f(z_j))\}_{i=1}^N$ is a martingale with respect to $\{\textbf{X}_i\}_{i=1}^N$ we directly obtain that $L_{\widetilde{\mathcal{D}}}(h)=\widetilde{L}_{\mathcal{D}}(h)$ which satisfies condition (\ref{spassump}).
\end{remark}

\begin{remark}
In the special case that $\{z_i\}_{i=0}^N$ is an irreducible and positive recurrent Markov chain with stationary distribution $\rm\pi$ and state space $\mathcal{S}$, we define $\tau_j^{(r)}$ as the $r$-th time the process returns to state $j$ and $\tau_j^{(0)}=0$ if $z_0=j$. Suppose that $z_0=j$, $g$ is a function $g: \mathcal{S} \rightarrow \mathbb{R}$ and let $\mu=\sum_{i \in \mathcal{S}} \pi_i g(i)$, $\sigma^2=\mathbb{E}_j[(\sum_{k=1}^{\tau_j^{(1)}}(g(z_k)-\mu))^2]$, then according to the central limit theorem \citep{bhattacharya2009stochastic} we obtain that
\begin{equation}
\frac{\sum\limits_{k=0}^N (g(z_k)-\mu)}{\sqrt{N+1}\sigma} \mathop{\longrightarrow}^{\mathcal{D}} \mathcal{N}(0,1).
\end{equation}
Consequently,
\begin{equation}
\begin{split}
&L_{\widetilde{\mathcal{D}}}(h)=\frac{1}{N}\sum_{i=1}^N \mathbb{E}[f(z_i)] \mathop{\longrightarrow}^{P} \sum_{i \in \mathcal{S}} \pi_i f(i),\\
&\widetilde{L}_{\mathcal{D}}(h)=\frac{1}{N}\sum_{i=1}^N \mathbb{E}[f(z_i) | z_j, j < i] \mathop{\longrightarrow}^{P} \sum_{i \in \mathcal{S}} \pi_i \mathbb{E}[f(z_1) | z_0=i]=\sum_{i \in \mathcal{S}} \pi_i \sum_{j \in \mathcal{S}} P_{ij}f(j)=\sum_{j \in \mathcal{S}} \pi_j f(j).
\end{split}
\end{equation}
Therefore, condition (\ref{spassump}) holds for large enough N.
\end{remark}

Note that for the bound of Theorem \ref{stob} to be nontrivial the convergence rate $\widetilde{c}$ is required to be not sufficiently large. For instance, if $\widetilde{c}=O(N)$ and $d_{\mathcal{G}}<\infty$, the excess risk is bounded by $O(\sqrt{\frac{1}{N}\ln{N}})$ apart from the constant term. In general, our bound is convergent under the natural assumption of Lemma \ref{spl2}.

\section{Experiments}\label{Sect4}

In this section, we use synthetic and real-world data to evaluate the performance of the proposed classifier with other methodologies (DuSK \citep{he2014dusk}, SVM, STM, MRMLKSVM \citep{gao2018multiple}, SMM \citep{Luo2015Support}), since they have been proven successful in various applications. We first introduce the date sets constructed and describe how we conduct the experiments. The settings of our experiments are introduced below.

All the parameters involved are selected via cross validation. For each setting we average results over 10 trials each of which is obtained from the proposed generation. The input matrices are converted into vectors when it comes to the SVM problems. All kernels select the optimal trade-off parameter $C \in \{10^{-2},10^{-1},\cdots,10^2\}$, kernel width parameter $\sigma \in \{10^{-4},10^{-3},\cdots,10^4\}$ and rank $r \in \{1,2,\cdots,10\}$. More details about the parameter selection will be discussed later in this section. In MRMLKSVM, DuSK and SVM, Gaussian RBF kernel $k(\textbf{x},\textbf{y})=\exp(-\sigma \|\textbf{x}-\textbf{y}\|^2)$ is used as the vector kernel function.

All experiments were conducted on a computer with Intel(R) Core(TM) i5 (3.30 GHZ) processor with 16.0 GB RAM memory. The algorithms were implemented in Matlab.

\subsection{Classification Performance on Synthetic Data}

In order to get better insight of the proposed approach, we focus on the behavior of proposed methods for various attributes in binary classification problems. We generate a synthetic data set of $2N$ examples as follows. We construct two groups of $n \times n$ matrices from the following distributions.
\begin{equation*}
\begin{split}
&\textbf{P}_{i,j} \overset{i.i.d}{\sim} \mathcal{N}(\textbf{0}, \bm{\Sigma}_i) , \\
&\bm{\Sigma}_i= \textbf{R}(a\textbf{Q}+b\bm{\varepsilon}_i)(a\textbf{Q}+b \bm{\varepsilon}_i)^{\intercal} \textbf{R},\\
& \textbf{R}=\diag{(c/{\delta_1},\cdots,c/{\delta_n})} , \\
& \textbf{X}_i=\textbf{P}_i \textbf{P}_i^{\intercal}, \\
\end{split}
\end{equation*}
for $i=1,2$, $j=1, \cdots n$, $\textbf{Q},\bm{\varepsilon}_1,\bm{\varepsilon}_2$ are $n \times n$ matrices of normally distributed random numbers, $[\delta_1^2,\cdots,\delta_n^2]=\diag((a\textbf{Q}+b\bm{\varepsilon}_i)(a\textbf{Q}+b \bm{\varepsilon}_i)^{\intercal})$ and $a,b,c \in \mathbb{R}^+$. In other words, The entries of each matrix \textbf{P} of two groups have the same variance $c^2$. The rows within the same group have the same covariance, while the rows between different groups have distinct covariance. We set $a=10$, $c=\sqrt{20}$ and $N=500$ for this simulation. We evaluate the classification accuracy by adding different levels of Gaussian noises from $b=2^{0.2}$ to 2 with fixed $n=20$, and bounding the size of examples from $n=15$ to 55 with fixed $b=2$. Then, we use another $M=5000$ examples for testing.
\begin{figure}
\centering
\subfigure[]{
\includegraphics[width=0.48\textwidth]{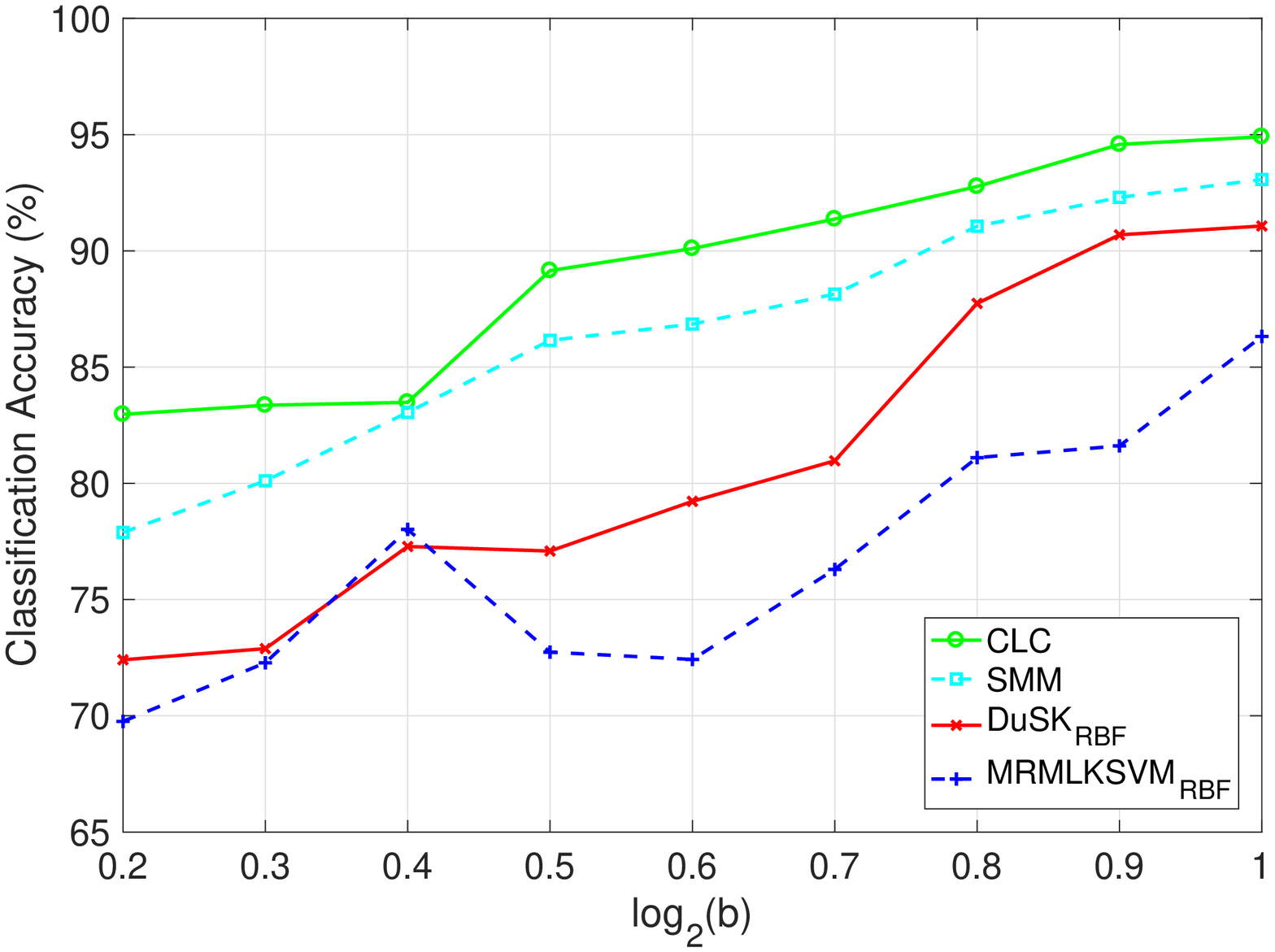}}
\subfigure[]{
\includegraphics[width=0.48\textwidth]{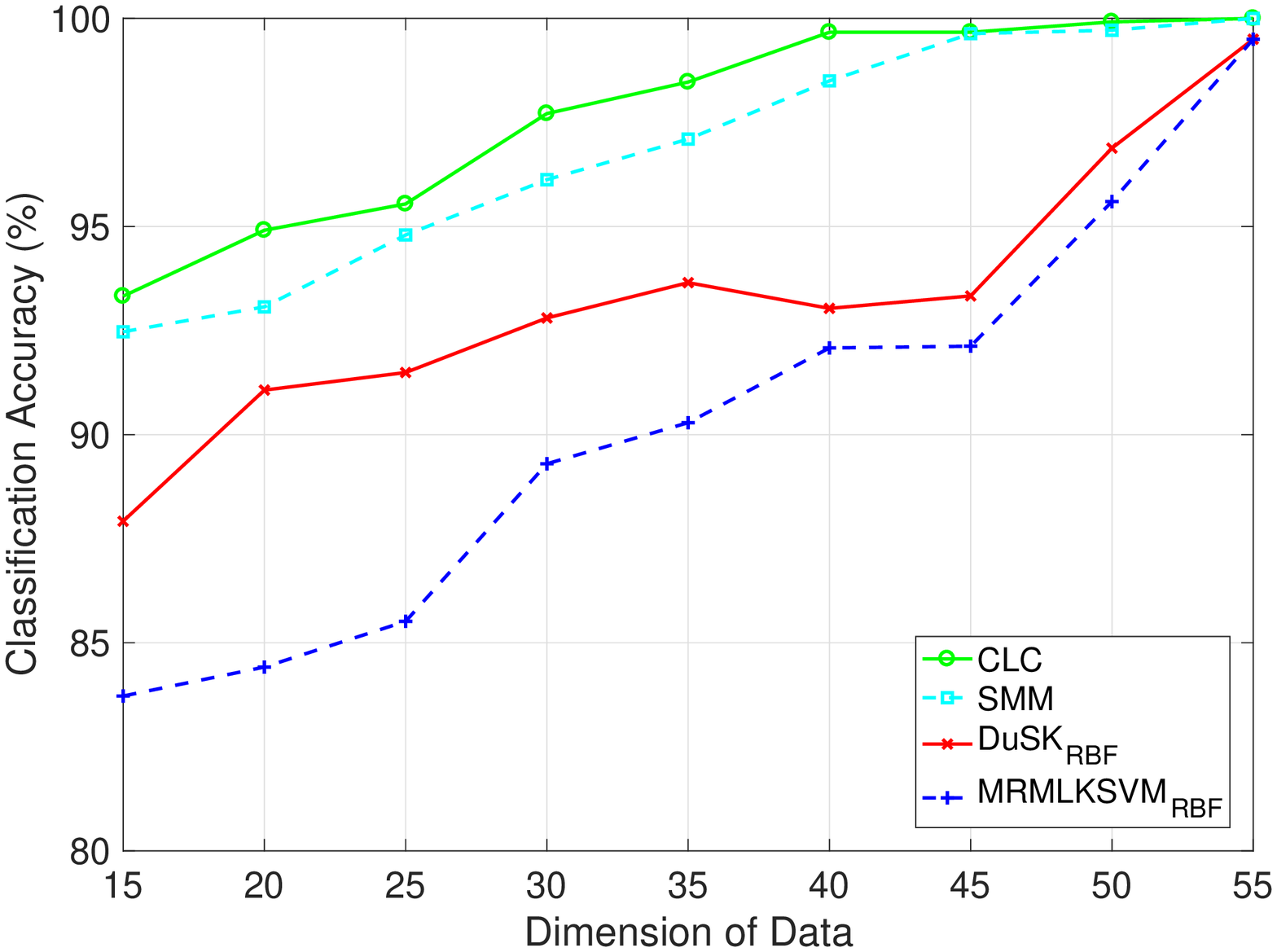}}
\caption{Classification accuracy on synthetic data with different levels of noises and sizes. We use Gaussian noise with 0 mean and standard derivation from $2^{0.2}$ to 2 in (a), and different size $n$ from $15 \times 15$ to $55 \times 55$ in (b)}
\label{Fig1}
\end{figure}

The results in Fig.~\ref{Fig1} show classification performance of compared methods. We can observe that the coefficient $b$ has a significant effect on the test accuracy. As $b$ increases, the covariance matrices of two groups of data become less similar to each other. The simulation results show that the proposed method is able to capture correlationship in the feature matrices. The dimension of the data matrices has a positive effect on the classification results. It is clear that all methods achieve comparable performance on synthetic data, but the proposed classifier is more robust with respect to variety levels of noises or dimensions.

\subsection{Classification Performance on Real-world Data}
Next, we evaluate the performance of proposed classifier on real data sets in the application of face and gesture classification. We consider the following benchmark datasets to perform a series of comparative experiments on multiple classification problems.
We use the IMM Face Dataset\footnote{\url{http://www.imm.dtu.dk/~aam/}}, the Japanese Female Facial Expression (JAFFE) Dataset \footnote{\url{http://www.kasrl.org/jaffe.html}}\citep{Lyons1998Coding} and the Jochen Triesch Static Hand Posture Dataset\footnote{\url{http://www.idiap.ch/resource/gestures/}}\citep{Triesch1996Robust}. Table \ref{Table2} summarizes the information for these three datasets. To better visualize the experimental data, we randomly choose a small subset for each database, as shown in Fig.~\ref{Fig2}.

\begin{table}
\caption{Statistics of datasets.}
\begin{tabular*}{1\textwidth}{@{\extracolsep{\fill}}  llll}
\noalign{\smallskip}
\hline
\noalign{\smallskip}
Dataset & \#Instances & \#Class & Size \\
\noalign{\smallskip}
\hline
\noalign{\smallskip}
IMM & 240 & 40 & $480\times640$ \\
JAFFE & 213 & 10 & $256\times 256$ \\
Jochen Triesch & 720 & 10 & $128\times 128$ \\
\noalign{\smallskip}
\hline
\label{Table2}
\end{tabular*}
\end{table}

\begin{figure}
\centering
\subfigure[]{
\includegraphics[width=0.32\textwidth,]{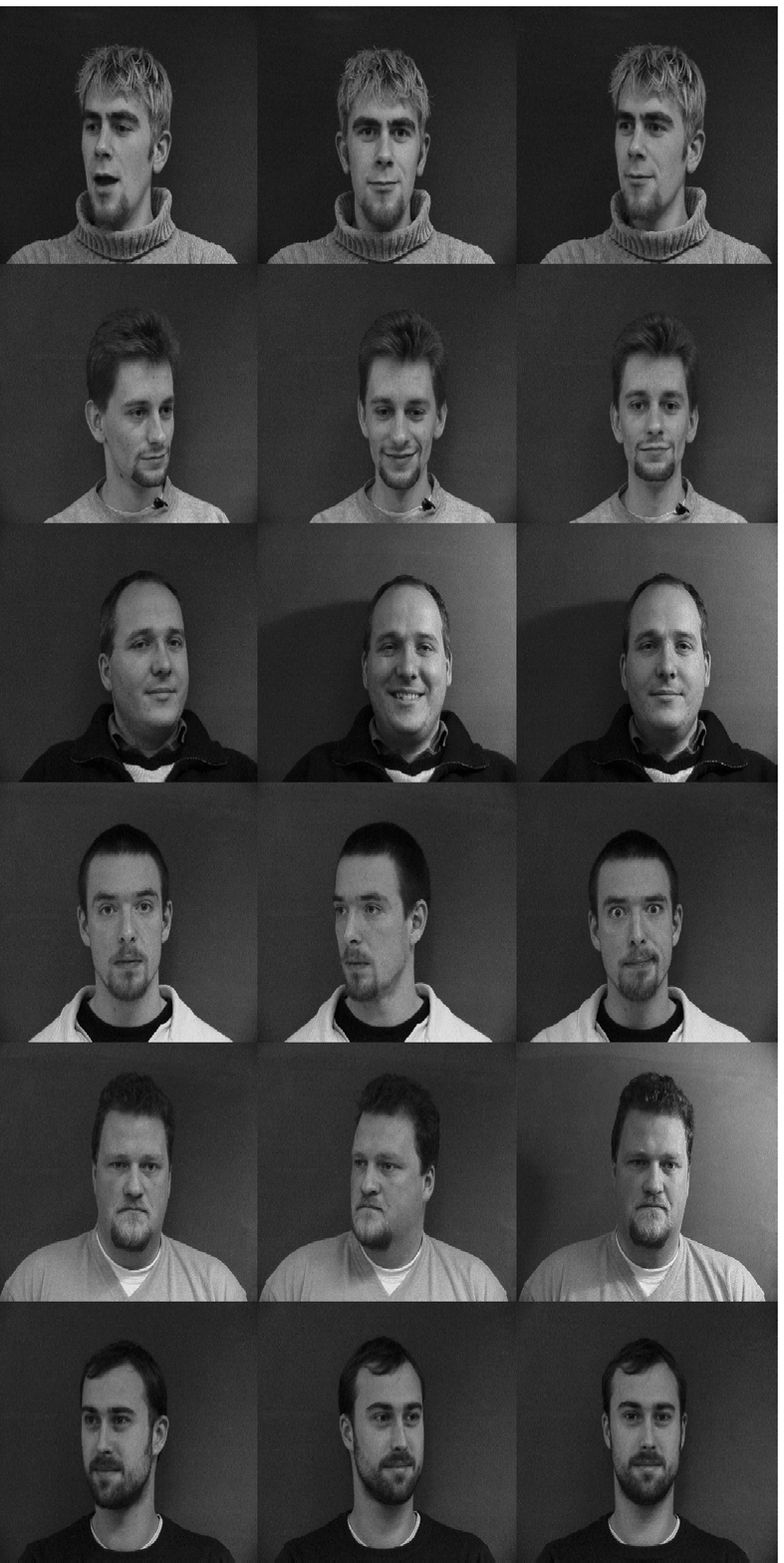}}
\subfigure[]{
\includegraphics[width=0.32\textwidth]{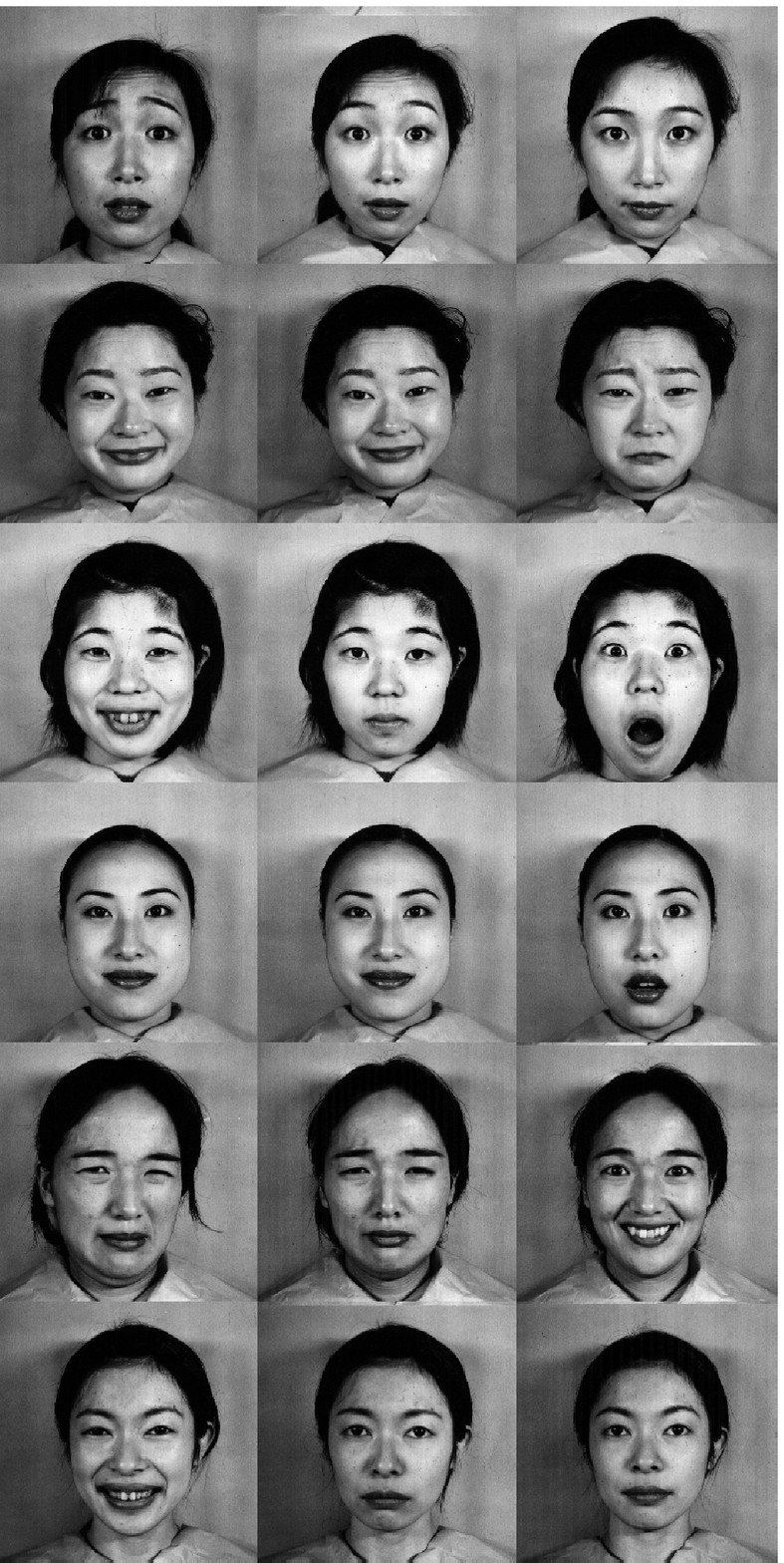}}
\subfigure[]{
\includegraphics[width=0.32\textwidth]{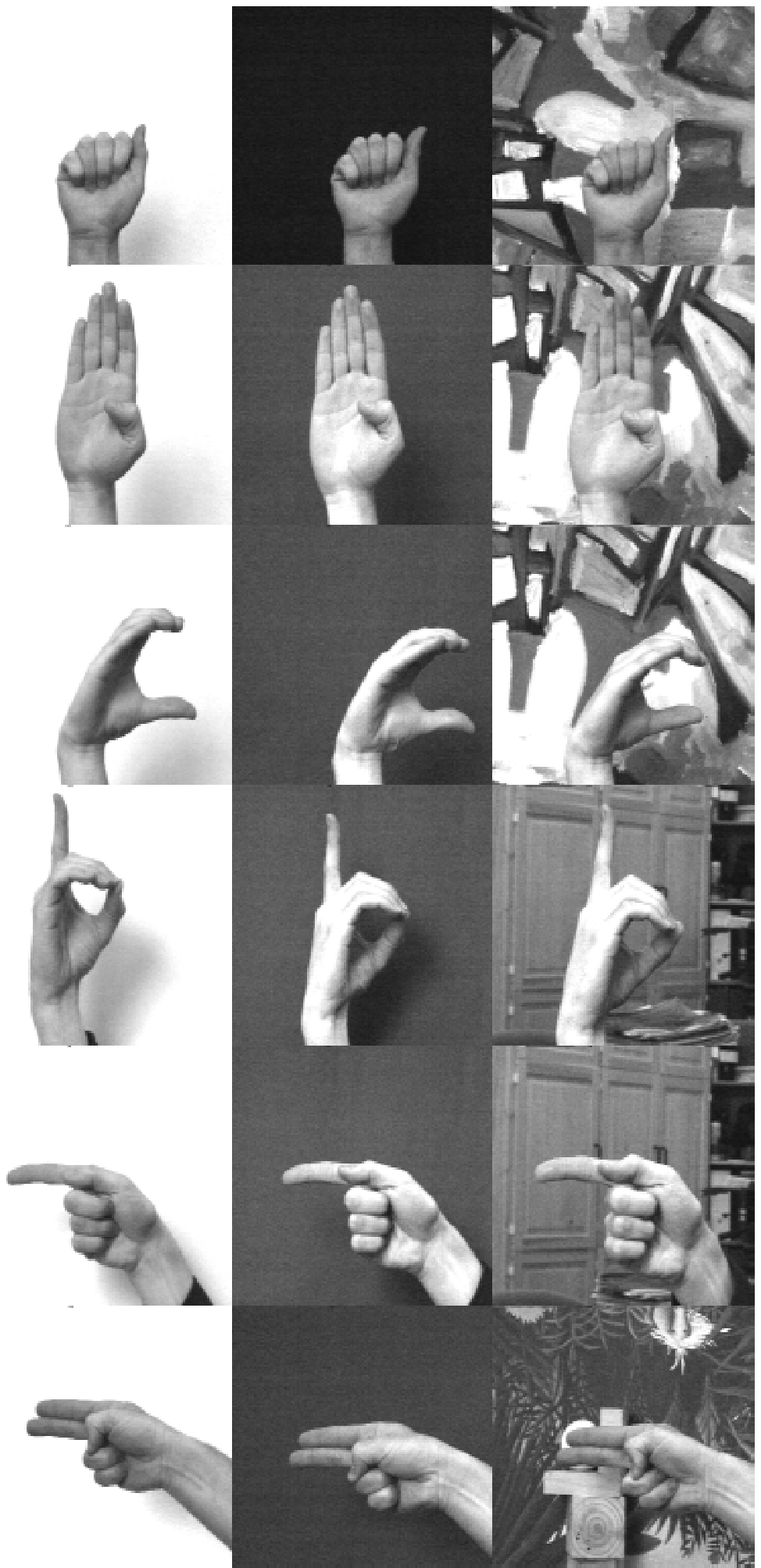}}
\caption{Example images for classification problems. \textbf{a} IMM,  \textbf{b} JAFFE, \textbf{c} Jochen Triesch Static Hand Posture}
\label{Fig2}
\end{figure}

The IMM Face Database comprises 240 still images of 40 different human faces, all without glasses. The gender distribution is 7 females and 33 males. To extract matrix-form data, we convert each color image into $24 \times 32$ and $48 \times 64$ gray level ones and use feature scaling to bring all values into the interval $[0,1]$.

The JAFFE Database contains 213 images of 7 facial expressions (e.g., happy, sad, angry, etc.) posed by 10 Japanese female models. Each image has been rated on 6 emotion adjectives by 60 Japanese subjects. We normalize the samples into $32 \times 32$ gray images as input matrices.

The Jochen Triesch Static Hand Posture Database consists of 10 hand signs performed by 24 persons against three backgrounds, uniform light, uniform dark and complex background. Images were recorded of size $128 \times 128$. We conduct our experiments on images against light and dark backgrounds separately.

We use half of examples for training. We use the pixel values as input matrices without any feature extraction techniques. It can be observed that when converting the input matrices into vectors, the dimension of each sample is much higher than the size of the training set, which makes the problem much more difficult.

\begin{table}
\caption{Prediction performance on three datasets in terms of accuracy (mean and standard deviation). }
\begin{tabular*}{1\textwidth}{@{\extracolsep{\fill}}  llllll}
\noalign{\smallskip}
\hline
\noalign{\smallskip}
Datasets& $\rm SVM_{RBF}$ \ \ & $\rm DuSK_{RBF}$ \ \ & $\rm MRMLKSVM_{RBF}$ & SMM \ \ & Ours  \\
\noalign{\smallskip}
\hline
\noalign{\smallskip}
IMM$24\times32$ &  73.50(2.49) & 78.83(3.56) & 73.83(3.99) & 83.17(3.04) & \textbf{84.83(2.07)} \\
IMM$48\times64$ &  69.83(1.99) & 79.83(1.85) & 77.50(2.47) & 84.33(2.26) & \textbf{85.83(0.91)} \\
JAFFE & 93.25(1.59) & 94.89(1.53) & 85.58(6.52) & 95.30(0.98) & \textbf{95.71(1.00)} \\
Jochen Triesch(light) & 51.50(5.23) & 61.50(2.76) & 38.67(4.43) & 62.33(2.00) & \textbf{63.50(2.14)} \\
Jochen Triesch(dark) & 36.74(5.15) & 53.68(2.21) & 36.75(2.90) & 58.38(2.51) & \textbf{60.57(3.00)} \\
\noalign{\smallskip}
\hline
\label{Table3}
\end{tabular*}
\end{table}

The results are presented in Table \ref{Table3}, where best results are highlighted in bold type. It can be observed that all matrix classification methods can beat the SVM approach, which indicates that leveraging the structural information of each matrix data is meaningful.
It is also clear that our proposed method outperforms other competitive ones in general. This is because we use rows and columns of each image to construct multiple-distance, and split two classes by measuring its relative importance. These results demonstrate that such process efficiently capture the intrinsic correlation of the matrix data.

\begin{figure}
\centering
\subfigure[]{
\includegraphics[width=0.45\textwidth,]{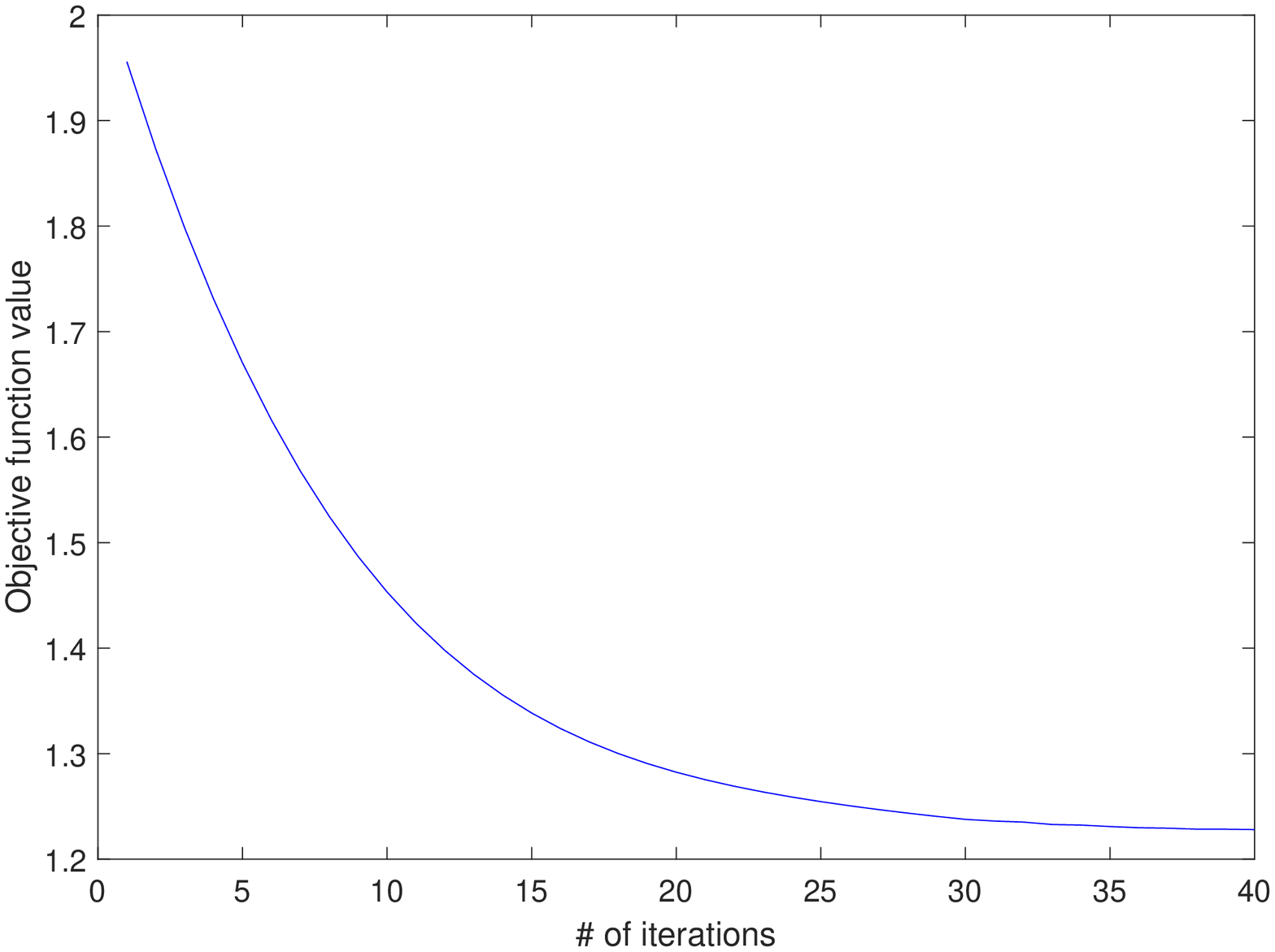}}
\subfigure[]{
\includegraphics[width=0.45\textwidth]{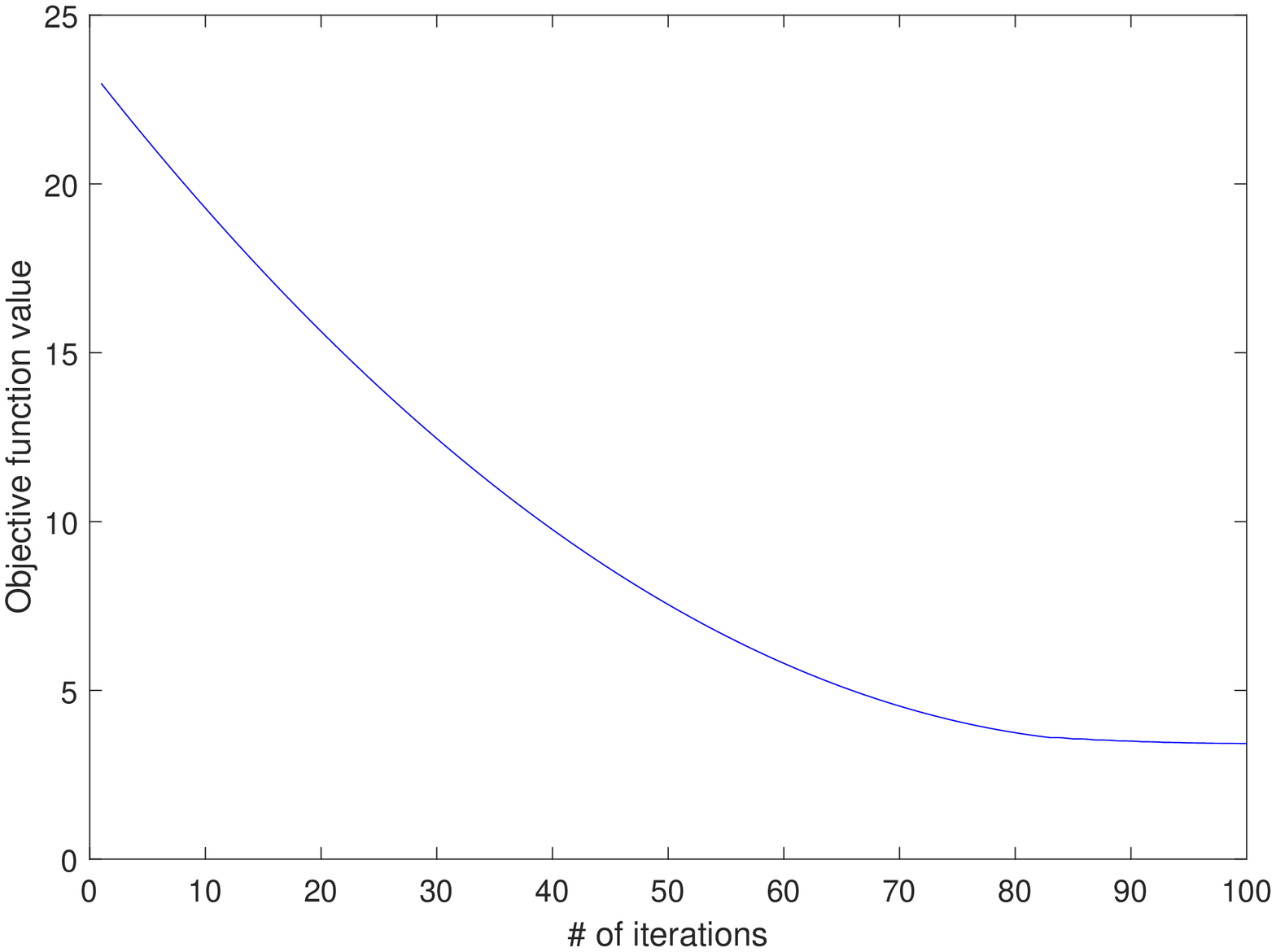}}
\subfigure[]{
\includegraphics[width=0.45\textwidth]{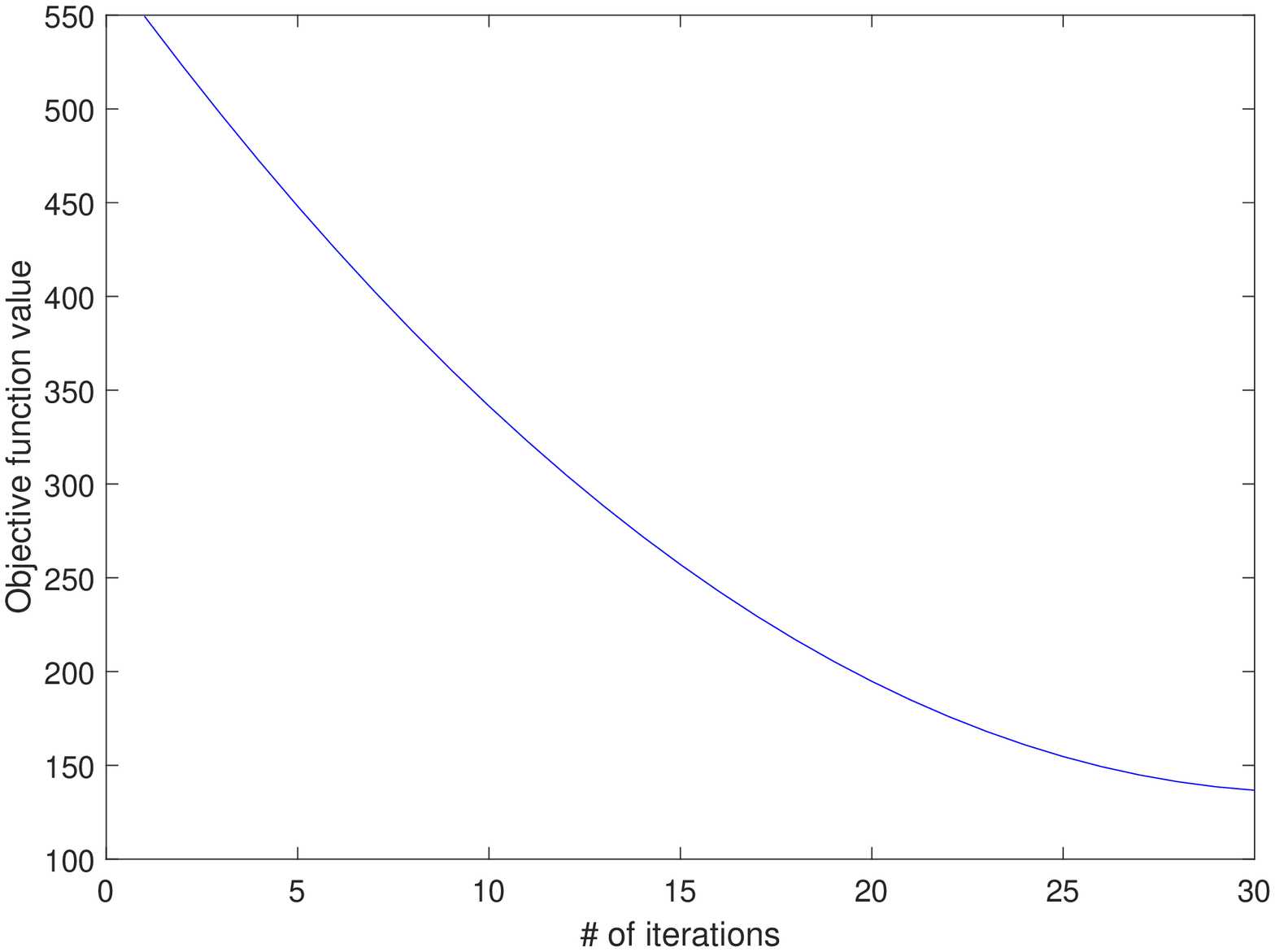}}
\subfigure[]{
\includegraphics[width=0.45\textwidth]{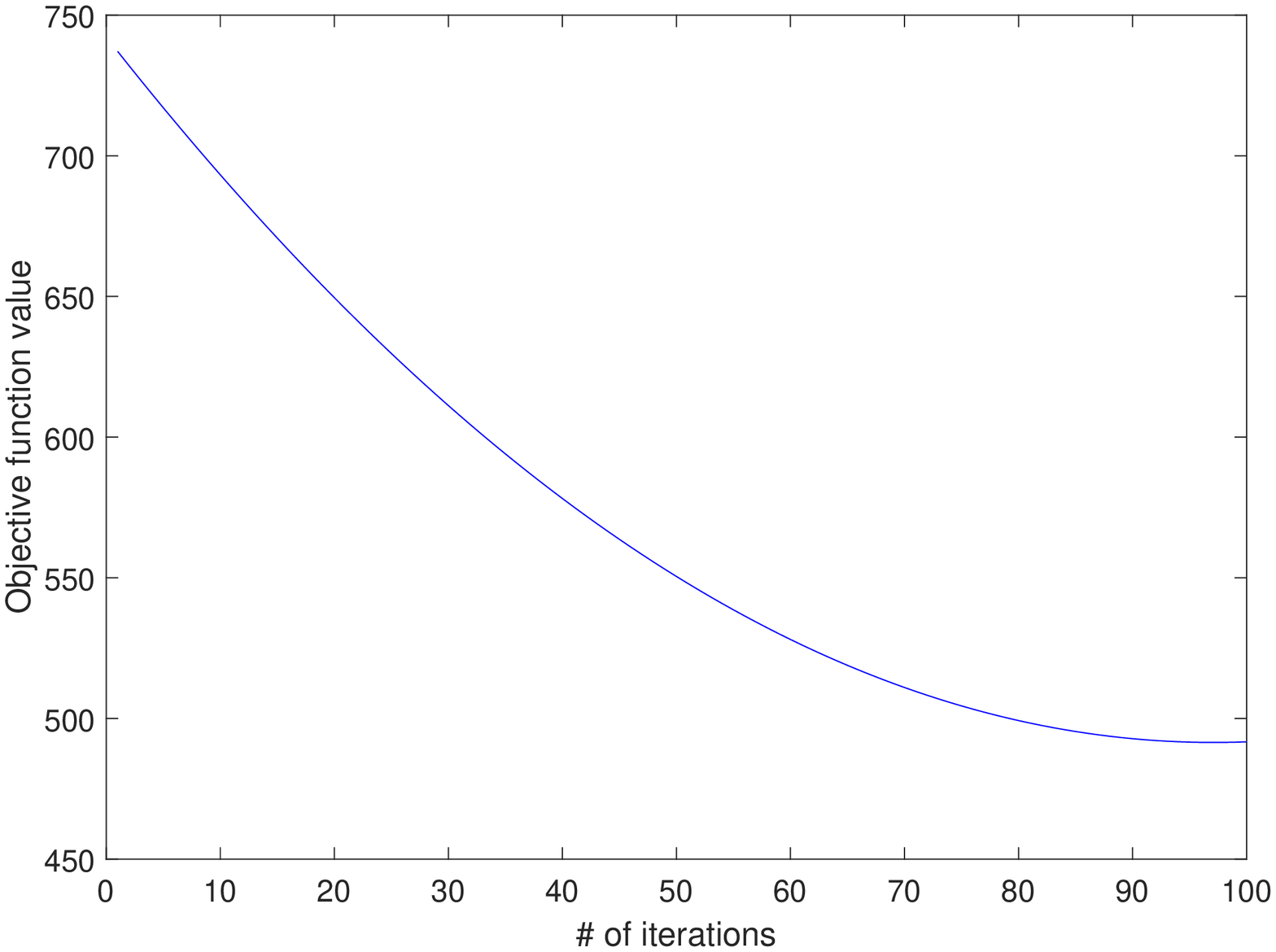}}
\caption{Convergence process on three datasets. \textbf{a} IMM$48 \times 64$,  \textbf{b} JAFFE, \textbf{c} Jochen Triesch(light), \textbf{d} Jochen Triesch(dark)}
\label{Fig3}
\end{figure}

Fig. \ref{Fig3} shows the convergence process of prosed method on three datasets. It exhibits fast convergence to the local optimal value in tens of iterations. It shows the efficiency of our method to train the classifier for real-word applications.

\begin{figure}
\centering
\subfigure[IMM$48 \times 64$]{
\includegraphics[width=0.23\textwidth,]{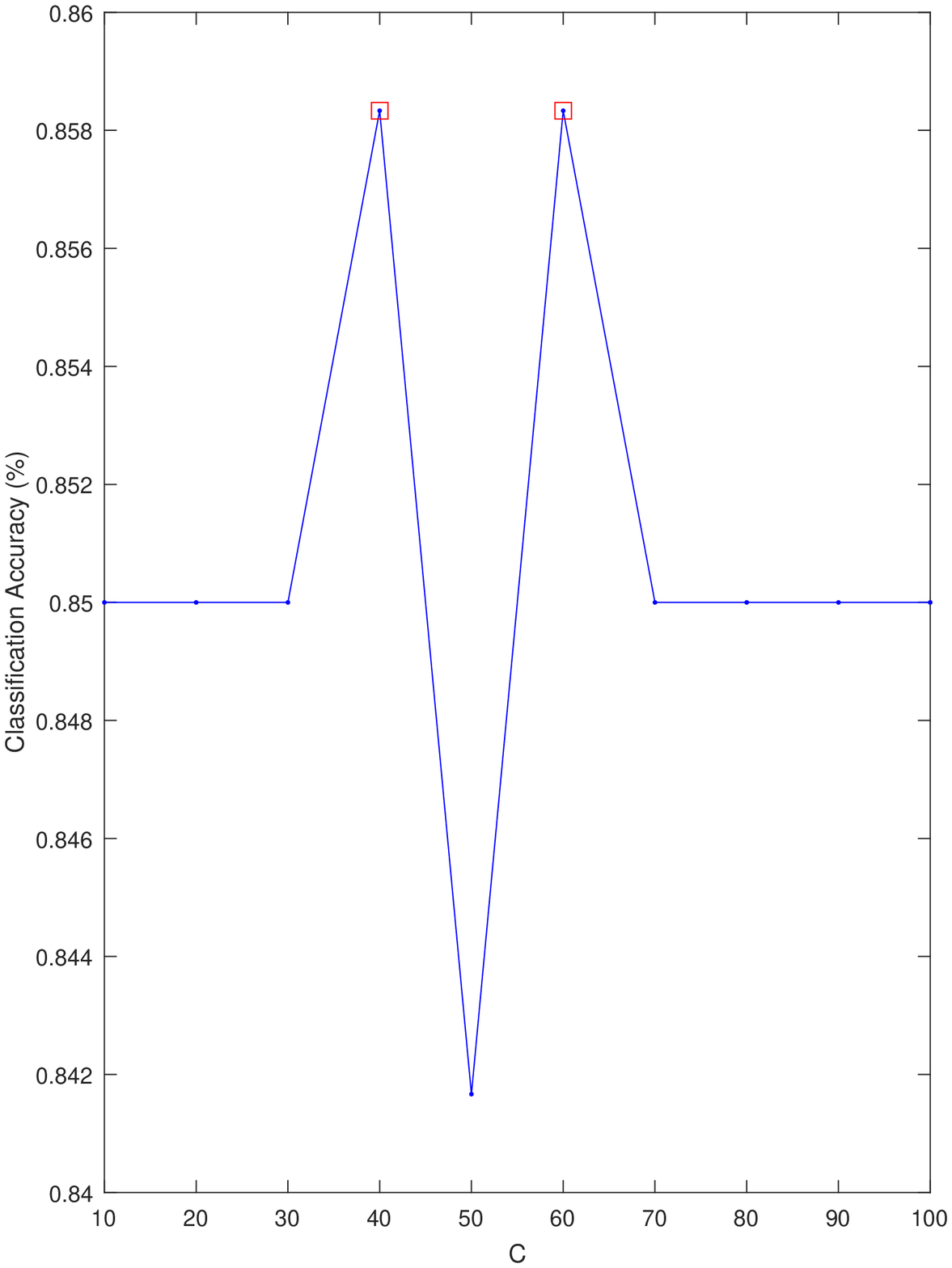}}
\subfigure[JAFFE]{
\includegraphics[width=0.23\textwidth]{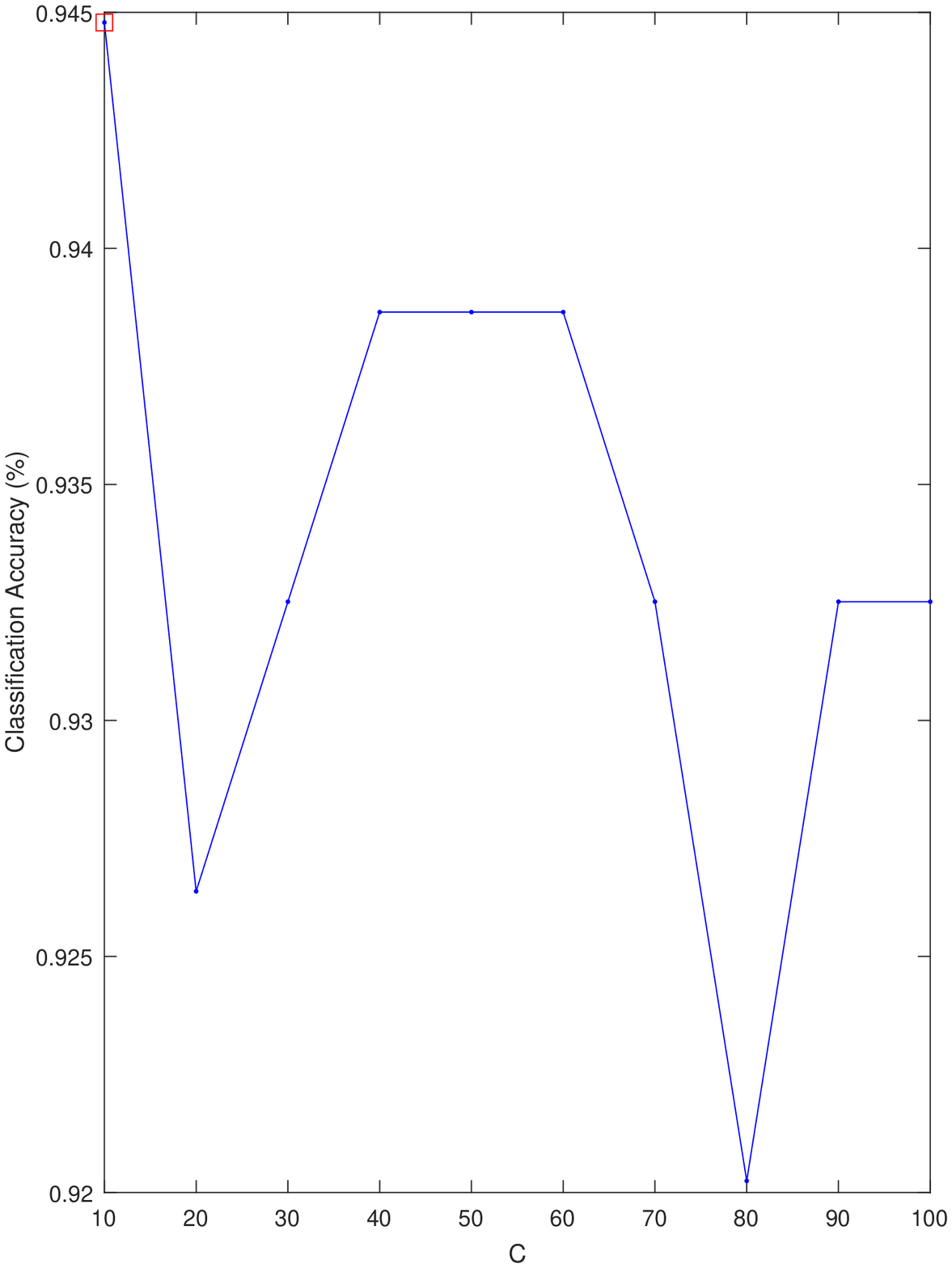}}
\subfigure[Jochen Triesch(light)]{
\includegraphics[width=0.23\textwidth,]{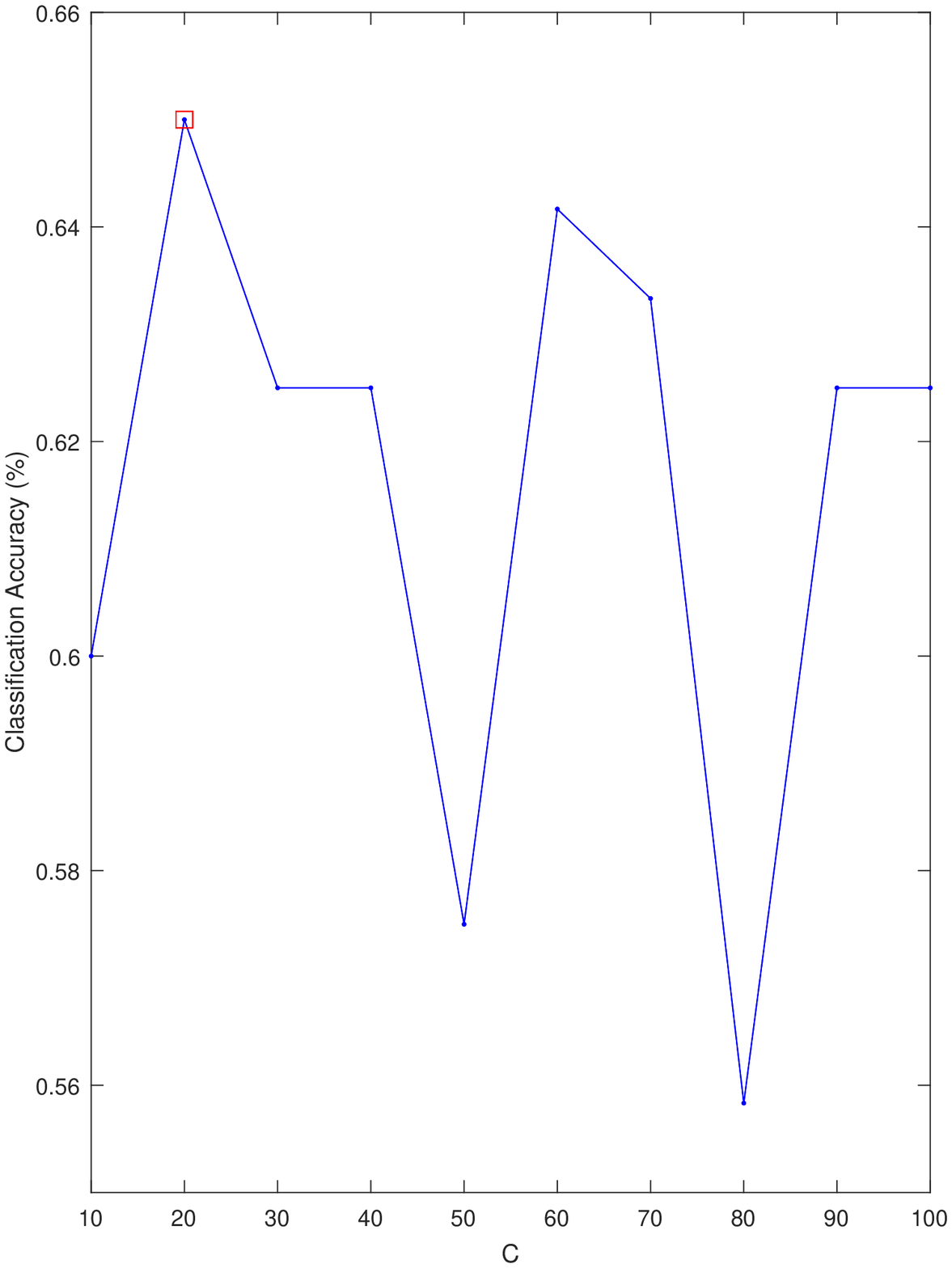}}
\subfigure[Jochen Triesch(dark)]{
\includegraphics[width=0.23\textwidth]{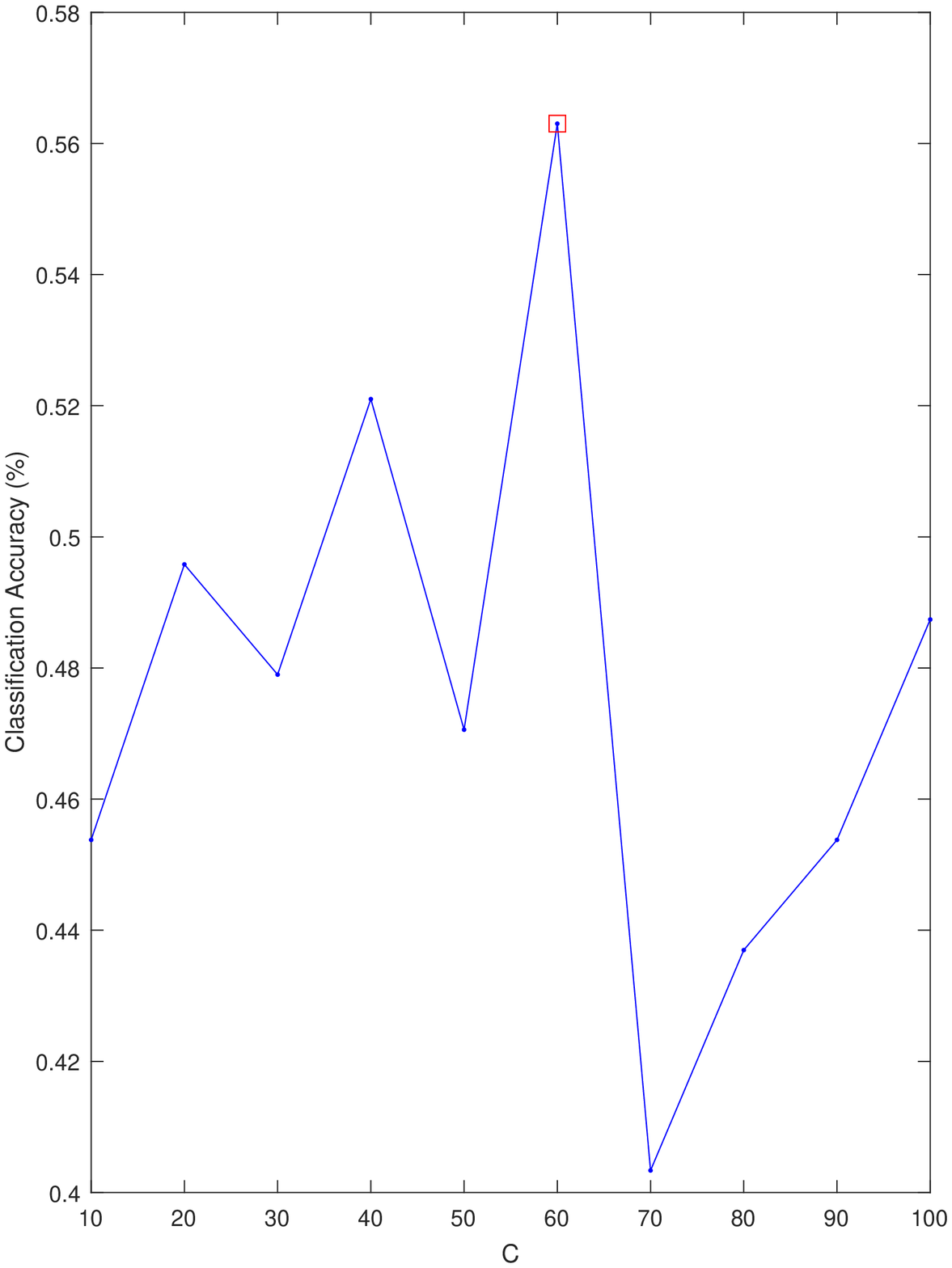}}
\caption{Classification accuracy compared by different $C$ on three data sets, where the red triangles indicate the optimal parameters.}
\label{Fig4}
\end{figure}

Although cross validation is used to search for the optimal trade-off parameter $C$, we are still interested in studying its sensitivity. The free parameter $C$ is supposed to balance the margin size and misclassification error. To this end, we conduct the experiments where $C \in \{10,20,\cdots,100\}$. Fig. \ref{Fig4} demonstrates that the parameter has a significant effect on the performance of classifier. Despite the fact that the optimal cost is changeable in different data sets, we can achieve best performance with appropriate $C$ via cross validation.

So far we have compared all experimental results of different baseline methods. Therefore, it can be concluded that the proposed classifier is a significantly effective and competitive alternative for real-world applications. Notice that rows and columns are introduced to construct the multi-distance, and other techniques can also be adopted.

\section{Concluding Remarks}\label{cr}
In this paper, we propose a multi-distance support matrix machine, which provides a principled way of solving matrix classification problems. The multi-distance is introduced to capture the correlation within matrix data, by means of an array which contains the products of columns and rows of the sample and the regression matrix. We also explore a weight function to measure the relative importance of the entries of the array. We further study the generalization bounds for i.i.d. processes and non i.i.d. process (stationary $\beta$-mixing, u.e.M.c. and martingale) based on SVM, SMM and MDSMM classifiers. We also present a more general approach for samples without prior knowledge. As our experimental results demonstrate, MDSMM is competitive in terms of accuracy with state-of-the-art classifiers on benchmark datasets.

In future work, we will seek technical solutions of (\ref{op}) to improve efficiency or figure out other approach to define multi-distance since the problem we analyze here is non-convex. We could only obtain a local optimal solution other than a global one which might deteriorate the performance in applications. Another interesting issue would be to take feature selection into consideration simultaneously, since the low-rank property and sparse properties are reasonable for regression matrix \citep{zheng2018sparse}.

\begin{acknowledgements}
The work is supported by National Natural Science Foundations of China under Grant 11531001 and National Program on Key Basic Research Project under Grant 2015CB856004.
\end{acknowledgements}

\section*{Appendix}
\begin{appendix}

\section{Proof of Theorem \ref{th1}}
\label{rA}
First, we recall some basic notations that are useful to our analysis.

The Rademacher complexity of $\mathcal{F}$ with respect to $\mathcal{S}$ is defined as follows:
\begin{equation*}
R(\mathcal{F} \circ \mathcal{S}) = \frac{1}{N} \underset{\bm{\sigma} \sim \{\pm1\}^N}{\mathbb{E}}\bigg[\sup\limits_{f \in \mathcal{F}} \sum_{i=1}^N \sigma_i f(z_i)\bigg].
\end{equation*}
More generally, given a set of vectors, $\mathcal{A} \subset \mathbb{R}^N$, we define
\begin{equation*}
R(\mathcal{A}) = \frac{1}{N} \underset{\bm{\sigma}}{\mathbb{E}}\bigg[\sup\limits_{\textbf{a} \in \mathcal{A}} \sum_{i=1}^N \sigma_i a_i\bigg].
\end{equation*}

In order to prove the theorem we rely on the generalization bound, we show the following lemmas to support our conclusion.

\begin{lemma}\label{l1}
Assume that for all z and $h \in \mathcal{H}_p$ we have that $|l(h,z)| \leq c$, then with probability at least $1-\delta$, for all $h \in \mathcal{H}_p$,
\begin{equation}
L_{\mathcal{D}}(h)-L_{\mathcal{S}}(h) \leq 2 \underset{\mathcal{S}' \sim D^N}{\mathbb{E}} R(\ell \circ \mathcal{H}_p \circ \mathcal{S}')+c\sqrt\frac{2\ln(2/\delta)}{N}.
\end{equation}
\end{lemma}

\begin{lemma}\label{l2}
For each $i=1,\cdots,N$, let $\Phi_i : \mathbb{R} \rightarrow \mathbb{R}$ be a $\rho$-Lipschitz function, namely for all $\alpha,\beta \in \mathbb{R}$ we have $|\Phi_i(\alpha)-\Phi_i(\beta)| \leq \rho |\alpha-\beta|$. For $\textbf{a} \in \mathbb{R}^N$, let $\Phi(\textbf{a})$ denote the vector $(\Phi_1(a_1), \cdots, \Phi_N(a_N))$ and $\Phi \circ \mathcal{A} = \{\Phi(\textbf{a}): \textbf{a} \in \mathcal{A}\}$. Then,
\begin{equation}
R(\Phi \circ \mathcal{A}) \leq \rho R(\mathcal{A}).
\end{equation}
\end{lemma}
The proof of Lemma \ref{l1} and \ref{l2} can be discovered in \cite{Shalev2014Understanding}. Additionally, we present the next lemma.

\begin{lemma}\label{l3}
Let $\mathcal{S}=(\emph{\textbf{X}}_1,\cdots,\emph{\textbf{X}}_N)$ be a finite set of matrices in a Hilbert space. We define $\mathcal{H} \circ \mathcal{S} = \{\emph{\textbf{w}}^{\intercal}\emph{\textbf{d}}_2(\emph{\textbf{X}}_1,\emph{\textbf{Z}}),\cdots,\emph{\textbf{w}}^{\intercal}\emph{\textbf{d}}_2(\emph{\textbf{X}}_N,\emph{\textbf{Z}}) : \|\emph{\textbf{w}}\| \leq 1, \|\emph{\textbf{Z}}\| \leq 1 \}$. Then,
\begin{equation}
R(\mathcal{H} \circ \mathcal{S}) \leq \frac{\max_i (\|\emph{\textbf{X}}_i \circ \emph{\textbf{X}}_i\|_1+\|\emph{\textbf{X}}_i \circ \emph{\textbf{X}}_i\|_{\infty})}{\sqrt{N}}.
\end{equation}
\end{lemma}
\begin{proof}
First, fix \textbf{Z} and use Cauchy-Schwartz inequality, we derive the following inequality
\begin{equation}
\begin{split}
N R(\mathcal{H} \circ \mathcal{S}) &= \underset{\bm{\sigma}}{\mathbb{E}} \bigg[\sup\limits_{\textbf{a} \in \mathcal{H} \circ \mathcal{S}} \sum_{i=1}^N \sigma_i a_i\bigg] \\
&= \underset{\bm{\sigma}}{\mathbb{E}} \bigg[\sup\limits_{\textbf{W}:\|\textbf{W}\| \leq 1} \sum_{i=1}^N \sigma_i \textbf{w}^{\intercal}\textbf{d}_2(\textbf{X}_i,\textbf{Z})\bigg] \\
&= \underset{\bm{\sigma}}{\mathbb{E}} \bigg[\sup\limits_{\textbf{W}:\|\textbf{W}\| \leq 1} \langle \textbf{w}, \sum_{i=1}^N \sigma_i \textbf{d}_2(\textbf{X}_i,\textbf{Z}) \rangle \bigg] \\
& \leq \underset{\bm{\sigma}}{\mathbb{E}} \bigg[ \| \sum_{i=1}^N\sigma_i \textbf{d}_2(\textbf{X}_i,\textbf{Z}) \|) \bigg].
\end{split}
\end{equation}

Next, use Jensen's inequality we have that
\begin{equation}
\underset{\bm{\sigma}}{\mathbb{E}} \bigg[ \| \sum_{i=1}^N\sigma_i \textbf{d}_2(\textbf{X}_i,\textbf{Z}) \| \bigg]=\underset{\bm{\sigma}}{\mathbb{E}} \bigg[ \Big(\| \sum_{i=1}^N\sigma_i \textbf{d}_2(\textbf{X}_i,\textbf{Z}) \|^2 \Big)^{1/2} \bigg] \leq \Big(\underset{\bm{\sigma}}{\mathbb{E}} \bigg[ \| \sum_{i=1}^N\sigma_i \textbf{d}_2(\textbf{X}_i,\textbf{Z}) \|^2 \bigg]\Big)^{1/2}.
\end{equation}

Since the variables $\sigma_1,\cdots,\sigma_N$ are independent we have
\begin{equation*}
\begin{split}
\underset{\bm{\sigma}}{\mathbb{E}} \bigg[ \| \sum_{i=1}^N\sigma_i \textbf{d}_2(\textbf{X}_i,\textbf{Z}) \|^2 \bigg] &= \underset{\bm{\sigma}}{\mathbb{E}} \bigg[  \sum_{i,j=1}^N \sigma_i \sigma_j \langle \textbf{d}_2(\textbf{X}_i,\textbf{Z}),\textbf{d}_2(\textbf{X}_j,\textbf{Z}) \rangle \bigg] \\
&=\sum_{i \neq j} \langle \textbf{d}_2(\textbf{X}_i,\textbf{Z}),\textbf{d}_2(\textbf{X}_j,\textbf{Z}) \rangle \underset{\bm{\sigma}}{\mathbb{E}} [\sigma_i \sigma_j] + \sum_{i=1}^N  \| \textbf{d}_2(\textbf{X}_i,\textbf{Z}) \|^2 \underset{\bm{\sigma}}{\mathbb{E}} [\sigma_i^2] \\
&= \sum_{i=1}^N \| \textbf{d}_2(\textbf{X}_i,\textbf{Z}) \|^2 \leq N \max_i (\|\textbf{X}_i \circ \textbf{X}_i\|_1+\|\textbf{X}_i \circ \textbf{X}_i\|_{\infty}) \|\textbf{Z}\|^2 \\
&=N \max_i (\|\textbf{X}_i \circ \textbf{X}_i\|_1+\|\textbf{X}_i \circ \textbf{X}_i\|_{\infty}) .
\end{split}
\end{equation*}

Combining these inequalities we conclude our proof.
\end{proof}
\qed

Finally, we complete our proof as follows. Let $\mathcal{F} = \{ (\textbf{X},y) \mapsto \Phi(\textbf{w}^{\intercal}\textbf{d}_2(\textbf{X},\textbf{Z}),y) : (\textbf{w},\textbf{Z}) \in \mathcal{H}_p\}$. Indeed, the set $\mathcal{F} \circ \mathcal{S}$ can be written as
\begin{equation*}
\mathcal{F} \circ \mathcal{S}=\{(\Phi(\textbf{w}^{\intercal}\textbf{d}_2(\textbf{X}_1,\textbf{Z}),y_1),\cdots,\Phi(\textbf{w}^{\intercal}\textbf{d}_2(\textbf{X}_N,\textbf{Z}),y_N)) : (\textbf{w},\textbf{Z}) \in \mathcal{H}_p\},
\end{equation*}
and $R(\mathcal{F} \circ \mathcal{S}) \leq \frac{\rho B D \sqrt{R_1+R_2}}{\sqrt{N}}$ with probability 1 follows directly by combining Lemma \ref{l2} and \ref{l3}. Then the claim of Theorem \ref{th1} follows from Lemma \ref{l1}.
\end{appendix}

\bibliographystyle{spbasic}
\bibliography{Ref_CLC}

\end{document}